\documentclass{article}

\usepackage[a4paper, total={6in, 9in}]{geometry}
\usepackage[round]{natbib}

\renewcommand{\bibname}{References}
\renewcommand{\bibsection}{\subsubsection*{\bibname}}

\usepackage{footmisc}
\usepackage{lmodern}
\usepackage{siunitx}
\usepackage{booktabs}
\usepackage{etoolbox}
\usepackage[dvipsnames]{xcolor}
\usepackage{amsmath,amsthm,amssymb,amsfonts}
\definecolor{pearThree}{HTML}{E74C3C}
\definecolor{pearcomp}{HTML}{B97E29}
\definecolor{pearDark}{HTML}{2980B9}
\definecolor{pearDarker}{HTML}{1D2DEC}

\usepackage{amsmath,amsfonts,bm}
\usepackage{hhline}

\usepackage{makecell}
\usepackage{thmtools}
\usepackage{thm-restate}
\usepackage{dsfont}
\usepackage{xspace}
\usepackage{empheq}
\usepackage{pifont}
\usepackage{natbib}
\usepackage{multicol}
\usepackage{tabularx}
\usepackage{fontawesome}
\usepackage{array, makecell} %
\usepackage{soul}
\usepackage{graphicx}
\usepackage{colortbl}
\usepackage{float}

\usepackage{rotating}

\usepackage{multirow}

\usepackage{amsthm}
\newtheorem{theorem}{Theorem}
\newtheorem{lemma}{Lemma}

\newtheorem{proposition}{Proposition}

\newtheorem{assumption}{Assumption}

\usepackage{enumitem}
\usepackage{subcaption}
\usepackage[ruled,vlined,noend]{algorithm2e}
\usepackage{dirtytalk}

\definecolor{pearThree}{HTML}{E74C3C}
\definecolor{pearcomp}{HTML}{B97E29}
\definecolor{pearDark}{HTML}{2980B9}
\definecolor{pearDarker}{HTML}{1D2DEC}
\usepackage{hyperref}
\hypersetup{
	colorlinks,
	citecolor=black,
	linkcolor=black,
	urlcolor=pearDarker}
\usepackage{cleveref}

\def\Secref#1{Section~\ref{#1}}

\def\eqref#1{equation~\ref{#1}}
\def\Eqref#1{Equation~\ref{#1}}

\def\1{\bm{1}}

\def\eps{{\epsilon}}

\DeclareMathAlphabet{\mathsfit}{\encodingdefault}{\sfdefault}{m}{sl}
\SetMathAlphabet{\mathsfit}{bold}{\encodingdefault}{\sfdefault}{bx}{n}

\def\gA{{\mathcal{A}}}

\newcommand{\R}{\mathbb{R}}

\newcommand{\wt}[1]{\widetilde{#1}}
\newcommand{\wb}[1]{\overline{#1}}

\newcommand{\wh}[1]{\widehat{#1}}

\definecolor{darkgreen}{rgb}{0,0.5,0}
\definecolor{darkred}{rgb}{0.7,0,0}
\definecolor{teal}{rgb}{0.3,0.8,0.8}

\makeatletter
\newcommand\footnoteref[1]{\protected@xdef\@thefnmark{\ref{#1}}\@footnotemark}
\makeatother

\usepackage{tikz}
\usetikzlibrary{arrows}

\tikzset{
  treenode/.style = {align=center, inner sep=0pt, text centered,
    font=\sffamily},
  arn_n/.style = {treenode, circle, white, font=\sffamily\bfseries, draw=black,
    fill=black, text width=1.5em},
  arn_r/.style = {treenode, circle, red, draw=red, 
    text width=1.5em, very thick},
  arn_x/.style = {treenode, rectangle, draw=black,
    minimum width=0.5em, minimum height=0.5em}
}

\usepackage{pgfplots}
\usepackage{tikz}
\pgfplotsset{compat=newest}
\usetikzlibrary{calc}
\usepackage[normalem]{ulem}
\usepackage{wrapfig}

\usepackage{authblk}

\title{Top $K$ Ranking for Multi-Armed Bandit with Noisy Evaluations}

\author[1,2]{Evrard Garcelon}
\affil[1]{Meta AI}
\affil[2]{CREST, ENSAE}
\author[3]{Vashist Avadhanula}
\author[1]{Alessandro Lazaric}
\author[1]{Matteo Pirotta}
\affil[3]{Core Data Science, Meta}
\date{}

\usepackage[toc,page,header]{appendix}
\usepackage{minitoc}

\begin{document}

\doparttoc 
\faketableofcontents 

\maketitle

%

%

\begin{abstract}
  We consider a multi-armed bandit setting where, at the beginning of each round, the learner receives noisy independent, and possibly biased, \emph{evaluations}
of the true reward of each arm and it selects $K$ arms with the objective of accumulating as much reward as possible over $T$ rounds. Under the assumption that at each round the true reward of each arm is drawn from a fixed distribution, we derive different algorithmic approaches and theoretical guarantees depending on how the evaluations are generated. First, we show a $\widetilde{O}(T^{2/3})$ regret in the general case when the observation functions are a genearalized linear function of the true rewards. On the other hand, we show that an improved $\widetilde{O}(\sqrt{T})$ regret can be derived when the observation functions are noisy linear functions of the true rewards. Finally, we report an empirical validation that confirms our theoretical findings, provides a thorough comparison to alternative approaches, and further supports the interest of this setting in practice.
\end{abstract}


\section{INTRODUCTION}\label{sec:intro}


Consider an idealized content reviewing task in a large social media firm, where the objective is to identify  harmful content that violates the platforms' community standards. Given the large volume of content generated on a daily basis, it may not be possible to ask human reviewers to provide a thorough assessment of each piece of content. For this reason, the platform may automatically assign a badness score for each piece of content depending on their estimated level of severity. For example, a hate speech related post may be assigned a higher badness score in comparison to a click bait post. The content with higher badness score may then be prioritized for human review, which eventually leads to what we can consider as a ``ground-truth'' evaluation of the severity of the content. The more accurate the badness score is in predicting the actual severity, the higher the chance that harmful content is passed for human review and properly identified. In practice, the badness score may be obtained by aggregating predictions returned by different automatic systems (e.g., rule-based, ML-based systems). For instance, the platform could rely on NLP-based classifiers for hostile speech detection, or CV-based classifiers for graphic images. As such, it is crucial to properly calibrate the predictions returned by each of these classifiers to ensure that the scores can be compared meaningfully and then return an aggregate and reliable badness score that correctly prioritizes the most harmful content for human review.
The problem sketched before\footnote{Similar problems are patient prioritization in hospitals~\citep{Dry2019PatientPT}, credit scoring~\citep{provenzano2020machine}, and resume review~\citep{li2020competencelevel}.} can be seen as an instance of the multi-armed bandit (MAB) framework, where each piece of content is an arm and the objective of the bandit algorithm is to select arms/content with the higher reward (e.g., severity). The algorithm can rely on the estimations returned by a set of \emph{evaluators} (e.g., a set of classifiers) to decide which arm to pull at each step (e.g., content to pass to human review). This setting can be formalized using a number of existing frameworks, such as MAB with expert advice, contextual bandit, bandit with side observation, and contextual bandit with noisy context. We postpone a thorough discussion about these models to Sect.~\ref{sec:related.work}.

In this paper, we consider the case where evaluators are noisy and possibly biased functions of the true reward of each arm. In particular, we consider two alternative settings, where evaluations are generated according to: \textbf{1)} a noisy generalized linear function of the true reward; \textbf{2)} a noisy linear function of the true reward. In both cases, we first define an ``oracle'' strategy that have prior knowledge of the evaluation function, including the noise distribution, and it is designed to maximize the rewards of the arms chosen at each round given the evaluations provided as input. We then devise the most suitable MAB strategies to approach the oracle's performance over time. In the first case, we show that one has to rely on an $\epsilon$-greedy strategy to avoid dependencies between the evaluations observed over time and the decisions taken by the algorithm. This eventually leads to a regret w.r.t.\ the oracle of order $\widetilde{O}(T^{2/3})$ over $T$ rounds. On the other hand, if the evaluation functions are linear and the variance of the additive noise is known, we show that a simple greedy strategy leveraging the specific structure of the problem is able to recover a $\widetilde{O}(\sqrt{T})$ regret. We then validate these results in a number of experiments. We first consider synthetic problems where we carefully design the MAB instances to support our theoretical findings and to compare to alternative approaches. Then we move to problems based on real data, where our assumptions may not be verified, to provide a more thorough evaluation of performance and robustness of our approach. Notably, we study a problem related to content review prioritization for integrity in social media platforms. 


\section{PRELIMINARIES}\label{sec:preliminaries}

We consider a multi-armed bandit problem where, at each round $t$, the learner is provided with a set of $K_t > K \geq 1$ arms (e.g., the content to be reviewed at time $t$%
) characterized by a reward $r_{i,t}\in\R$ for each $i=1,\ldots,K_t$ (e.g., the badness score%
). While the true reward is unknown to the learner, $J$ evaluators return noisy, possibly biased, evaluations $f_{i,t,j}$ for each arm (e.g., different rule-based and/or ML-classifiers
). At the beginning of each round, the learner receives the evaluations $\{f_{i,t,j}\}$, it returns a set $\gA_t\subseteq \{1,\ldots,K_t\}$ of $K$ arms (i.e., $|\gA_t|=K$), and it observes their associated rewards $r_{i,t}$ for $i\in\gA_t$.\footnote{For the sake of simplicity, we consider that the learner receives the exact reward $r_{i,t}$, but all our results can be adapted to the case of noisy feedback.} The learner's objective is to accumulate as much reward as possible over $T$ rounds by selecting the $K$ arms with larger rewards.

Without any further assumption, this problem is not tractable since the rewards may change arbitrarily over time and the evaluations may not be predictive of the true rewards, thus making it impossible for any learner to achieve a satisfactory performance. Throughout the paper, we make a series of assumptions to make the problem solvable. We start from the  rewards.

\begin{assumption}\label{asm:rewards}
	The rewards $r_{i,t}$ of each arm $i=1,\ldots,K_t$ at round $t=1,\ldots,T$ are drawn i.i.d.\ from a common distribution $\nu$ supported on $[0,C]$ with $C$ a positive constant. The number of arms $K_t$ at each round $t$ is arbitrary and $K < K_t \leq K_{\max} < \infty$.
\end{assumption}

While this assumption simplifies the treatment of the problem, it does not affect the objective of the learner, which is to select the top-K arms \emph{at each round}, i.e., for the specific realizations $\{r_{i,t}\}$. 
For instance, for $K=1$, the objective is to return the arm $i_t^\star = \arg\max_{i=1,\ldots,K_t} r_{i,t}$. We do not assume that the learner has any prior knowledge of the distribution $\nu$. 

In general, the evaluators may rely on some contextual information $x_{i,t}$ associated to each arm $i$ (e.g., texts, images, meta-data related to the piece of content) to return their evaluation $f_{i,t,j}$ and their accuracy in predicting the true reward $r_{i,t}$ may vary depending on the evaluator and the specific context $x_{i,t}$. Nonetheless, we assume that the learner has no access to  the context or the actual mechanism that generates the evaluations (e.g., the evaluators may be external services) and we rather rely on the following model to describe how the evaluations are generated
\begin{align}\label{eq:gen.model}
	f_{i,t,j} = f_j(r_{i,t}) + \epsilon_{i,t,j}, \enspace j = 1,\ldots,J,
\end{align}
where $f_j:\R\rightarrow\R$ is the evaluation function, and $\epsilon_{i,t,j}$ is a stochastic error.
This general formulation can be seen as the inverse of a calibration function, as it describes the intrinsic bias of each evaluator $j$ and the noise associated to the evaluations of the true reward. See Fig.~\ref{foot:evaluation.function} for a qualitative illustration of this model. We assume the noise in the evaluations satisfy a rather mild assumption.\footnote{A similar assumption is used in~\citep{yun2017contextual}, where the covariance matrix of the distribution generating the noisy features is assumed to be known to the learning.}

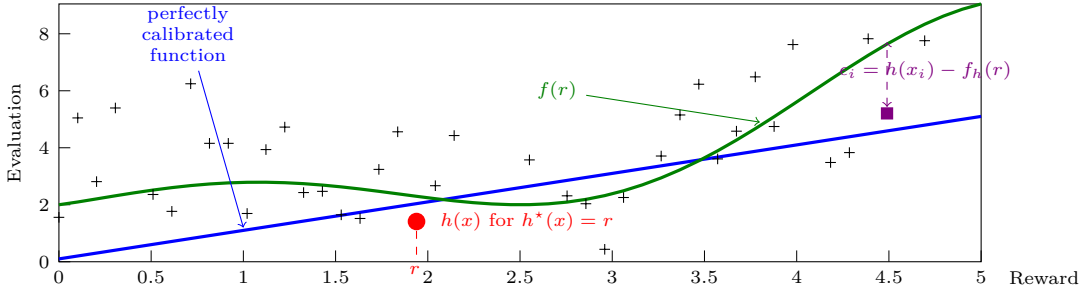
\begin{figure}[t]
\centering
\begin{tikzpicture}

\definecolor{color0}{rgb}{0.501960784313725,0,0.501960784313725}

\begin{axis}[
tick pos=left,
x grid style={white!69.0196078431373!black},
xlabel={Reward},
xmin=0, xmax=5,
xtick style={color=black},
y grid style={white!69.0196078431373!black},
ylabel={Evaluation},
ymin=0, ymax=9.04771365852811,
ytick style={color=black},
width=.9\columnwidth,
height=5cm,
clip=false,
font=\scriptsize,
x label style={at={(axis cs:5.1,0)},anchor=north west},
]
\addplot [draw=black, fill=black, mark=+, only marks]
table{%
x  y
0 1.55825554601899
0.102040816326531 5.04715262305068
0.204081632653061 2.81159258198463
0.306122448979592 5.39612407437904
0.510204081632653 2.35985039131402
0.612244897959184 1.76808878167125
0.714285714285714 6.24595074387002
0.816326530612245 4.15541999527149
0.918367346938776 4.15498269514821
1.02040816326531 1.693217476795
1.12244897959184 3.93649437289228
1.22448979591837 4.72496751500709
1.3265306122449 2.42800615204831
1.42857142857143 2.4701782457537
1.53061224489796 1.64626008966942
1.63265306122449 1.51791037445779
1.73469387755102 3.24095212674181
1.83673469387755 4.56016033896967
1.93877551020408 1.41204217773394
2.04081632653061 2.6681365450304
2.14285714285714 4.42847021608788
2.55102040816327 3.57402741963394
2.75510204081633 2.31375580262556
2.85714285714286 2.03558685179281
2.95918367346939 0.436430994512748
3.06122448979592 2.25402486850015
3.26530612244898 3.7101445105084
3.36734693877551 5.15107733320291
3.46938775510204 6.2292105803502
3.57142857142857 3.60823416643001
3.6734693877551 4.58153026052869
3.77551020408163 6.48681122758252
3.87755102040816 4.74547308379698
3.97959183673469 7.62157454361269
4.18367346938776 3.48338714205722
4.28571428571429 3.83037033380342
4.38775510204082 7.82655271125471
4.48979591836735 5.2055655115197
4.69387755102041 7.75707165070733
};
\addplot [very thick, blue]
table {%
0 0.1
0.102040816326531 0.202040816326531
0.204081632653061 0.304081632653061
0.306122448979592 0.406122448979592
0.408163265306122 0.508163265306122
0.510204081632653 0.610204081632653
0.612244897959184 0.712244897959184
0.714285714285714 0.814285714285714
0.816326530612245 0.916326530612245
0.918367346938776 1.01836734693878
1.02040816326531 1.12040816326531
1.12244897959184 1.22244897959184
1.22448979591837 1.32448979591837
1.3265306122449 1.4265306122449
1.42857142857143 1.52857142857143
1.53061224489796 1.63061224489796
1.63265306122449 1.73265306122449
1.73469387755102 1.83469387755102
1.83673469387755 1.93673469387755
1.93877551020408 2.03877551020408
2.04081632653061 2.14081632653061
2.14285714285714 2.24285714285714
2.24489795918367 2.34489795918367
2.3469387755102 2.4469387755102
2.44897959183673 2.54897959183673
2.55102040816327 2.65102040816327
2.6530612244898 2.7530612244898
2.75510204081633 2.85510204081633
2.85714285714286 2.95714285714286
2.95918367346939 3.05918367346939
3.06122448979592 3.16122448979592
3.16326530612245 3.26326530612245
3.26530612244898 3.36530612244898
3.36734693877551 3.46734693877551
3.46938775510204 3.56938775510204
3.57142857142857 3.67142857142857
3.6734693877551 3.7734693877551
3.77551020408163 3.87551020408163
3.87755102040816 3.97755102040816
3.97959183673469 4.07959183673469
4.08163265306122 4.18163265306122
4.18367346938776 4.28367346938775
4.28571428571429 4.38571428571429
4.38775510204082 4.48775510204082
4.48979591836735 4.58979591836735
4.59183673469388 4.69183673469388
4.69387755102041 4.79387755102041
4.79591836734694 4.89591836734694
4.89795918367347 4.99795918367347
5 5.1
};
\addplot [very thick, green!50.1960784313725!black]
table {%
0 2
0.102040816326531 2.09706130312208
0.204081632653061 2.20549677958852
0.306122448979592 2.31441184631055
0.408163265306122 2.41900221232157
0.510204081632653 2.51556451028029
0.612244897959184 2.60100821118818
0.714285714285714 2.67275971208367
0.816326530612245 2.72876183761616
0.918367346938776 2.76750051675188
1.02040816326531 2.78803589996756
1.12244897959184 2.79002837188594
1.22448979591837 2.77375477081605
1.3265306122449 2.74011230198898
1.42857142857143 2.69060880423858
1.53061224489796 2.62733878299982
1.63265306122449 2.55294518091751
1.73469387755102 2.47056731429788
1.83673469387755 2.38377579875992
1.93877551020408 2.29649563701776
2.04081632653061 2.21291895127065
2.14285714285714 2.13740911267431
2.24489795918367 2.07439824902275
2.3469387755102 2.02828029642931
2.44897959183673 2.00330189872978
2.55102040816327 2.00345354717546
2.6530612244898 2.03236339102111
2.75510204081633 2.09319613592801
2.85714285714286 2.18855938169155
2.95918367346939 2.32041963462133
3.06122448979592 2.49003006488289
3.16326530612245 2.69787186815195
3.26530612244898 2.9436108378601
3.36734693877551 3.22607046381709
3.46938775510204 3.54322255055763
3.57142857142857 3.89219600054832
3.6734693877551 4.26930404014491
3.77551020408163 4.67008978710533
3.87755102040816 5.08938967504634
3.97959183673469 5.52141387015781
4.08163265306122 5.95984244645767
4.18367346938776 6.39793573545157
4.28571428571429 6.82865694154828
4.38775510204082 7.2448048228431
4.48979591836735 7.63915398422725
4.59183673469388 8.00460012182929
4.69387755102041 8.33430739935798
4.79591836734694 8.62185503189766
4.89795918367347 8.8613801040163
5 9.04771365852811
};
\addplot [semithick, red, mark=*, mark size=3, mark options={solid}]
table {%
1.93877551020408 1.41204217773394
};
\addplot [semithick, color0, mark=square*, mark size=2, mark options={solid}]
table {%
4.48979591836735 5.2055655115197
};
\node[green!50.1960784313725!black] (fr) at (axis cs:2.7,6) {$f(r)$};
\draw[->,green!50.1960784313725!black] (fr) -- (axis cs:3.8,4.9);

\node[blue, text width=2cm, align=center] (cf) at (axis cs:0.7,8) {perfectly calibrated function};
\draw[->, blue] (cf) -- (axis cs:1,1.2);

\node (rp) at (axis cs:1.9387755102040818, 1.412042177733937) {};
\node[anchor=west,red] at (axis cs:1.9387755102040818, 1.412042177733937) {~~$h(x)$ for $h^\star(x) = r$};
\draw[red, dashed] (rp) -- (axis cs:1.9387755102040818,0);
\node[red] at (axis cs: 1.93, -0.3) {$r$};

\node (di) at (axis cs:4.48979591836735, 5.1055655115197) {};
\draw[<->, dashed, color0] (di) -- (axis cs: 4.48979591836735, 7.7);
\node[color0, anchor=west] at (axis cs: 4.18, 6.7) {$\epsilon_{i} = h(x_i)-f_h(r)$};
\end{axis}

\end{tikzpicture}
\caption{\small As an illustrative example, consider the case where at each round $t$ each arm is associated to a context $x_{i,t}$ drawn from a context distribution $\rho$ and there exists a function generating the true rewards as $r_{i,t} = h^\star(x_{i,t})$. The distribution $\nu$ is then defined by the distribution on rewards $r = h^\star(x)$ induced by $x\sim\rho$. We also denote by $\rho|h^\star(x)=r$ the conditional distribution over contexts associated with reward $r$. Consider then a neural network $h$, trained on past context-reward pairs, that returns an evaluation for arm $i$ characterized by a context $x_{i,t}$ as $f_{i,t} = h(x_{i,t})$. The black crosses in the plot are the pairs $(h^\star(x_{i,t}), h(x_{i,t}))$. The evaluation function associated to the neural network $h$ is then defined as $f_h(r) = \mathbb{E}_{x\sim\rho|h^\star(x)=r} \big[ h(x)\big]$ (green line) and the noise $\epsilon$ is the deviation from $h(x)$ and $f(r)$ depending on the specific realization of $x$, i.e., $\epsilon_{i,t} = h(x_{i,t}) - f_h(r)$ for $r = h^\star(x_{i,t})$. The blue line illustrates the perfectly calibrated case, where $h^\star$ itself is used for prediction, in this case $f_{h^\star}(r) = r$. \label{foot:evaluation.function}}
\end{figure}

\begin{assumption}\label{asm:noise}
Each error $\epsilon_{i,t,j}$ is generated i.i.d.\ from a sub-Gaussian distribution with zero mean and parameter $\sigma_j$, assumed to be known to the learner.
\end{assumption}


As far as the evaluation function is concerned, we distinguish two settings.

\begin{assumption}[Generalized linear  setting]\label{asm:general.setting}
	We assume that each evaluation function is a generalized linear model w.r.t.\ the true reward, i.e., $f_{j}(r) = g(\alpha_{j}\cdot r)$, for all $j\leq J$, where $\alpha \in \mathbb{R}$ and
	$g$ is a strictly increasing function and twice- differentiable with $\| g'\|_{\infty}\leq L_{g}$ and $\| g''\|_{\infty} \leq M_{g}$, and
	$c_{g} := \inf_{x, \theta} g'(x\cdot \theta) > 0$. The function $g$ is known to the learner, while the evaluator-specific parameters $\alpha_j$ are unknown.
\end{assumption}

\begin{assumption}[Linear setting]\label{asm:affine.setting}
	We assume that each evaluation function is linear w.r.t.\ the true reward, i.e., $f_j(r) = \alpha_j r$, for all $j\leq J$. While the shape of the function is known to the learner, the evaluator-specific parameters $\alpha_j$ are unknown.
\end{assumption}

In the following we use $\alpha = (\alpha_1, \ldots, \alpha_J)\in\R^J$ and $\sigma = (\sigma_1, \ldots, \sigma_J)\in\R^J$. 
We use the standard notation $\|\cdot\|$ and $\|\cdot\|_\infty$ for the $\ell_2$ and the maximum norm respectively, while for any two vectors $x,y\in\R^J$, $x \cdot y \in \R^J$ denotes the component-wise product.

We consider a setting where the bandit algorithm has only access to the predictors’ evaluations of the true reward of an arm. This is a very generic scenario that encompasses the case where evaluations are
a function of a context characterizing an arm. The generality of our framework allows us to deal with problems where
the context is not directly observable (e.g., because it is kept private) or where it differs across evaluators. For example,
similarly to bandits with expert advice, evaluators may use very different context sources (e.g., visual information,
text, meta data) to build their predictions, but these are unknown to the bandit algorithm (e.g., because evaluators
are external services). The resulting model in Eq.~\ref{eq:gen.model} is then a calibration function which, in the case evaluations are function of a context, can be understood as explaining the connection between the true reward and the evaluations after averaging over the stochasticity in the (non-observable) context information (Fig.~\ref{foot:evaluation.function}). Notice that if the context was available, the evaluations could be disregarded as the bandit could directly rely on the context to predict the rewards in the first place, as in standard contextual bandit.

We conclude by noticing, despite these additional assumptions, no learner can retrieve the best choice of top-K arms at each round (i.e., $\max_{i_{1}, \dots, i_{K}} \sum_{l=1}^{K} r_{i_{l},t}$), since the only information available to the learner is from biased and noisy evaluators (see Lemma~\ref{lem:lower.bound} in App.~\ref{app:lower.bound}). As a result, instead of targeting the top-K arms, in the following we introduce \emph{oracle} strategies that leverage the full knowledge of the problem (i.e., the evaluation function and the noise distribution) and use their performance as reference for the learner.

\vspace{-0.1in}
\section{RELATED WORK}\label{sec:related.work}
\vspace{-0.1in}

Before diving into how to solve the problem introduced in the previous section, we review alternative models that are related to our setting.
Let consider the case with $K=1$ (i.e., the learner returns one arm at each round). The most direct way to model our setting is MAB with expert advice~\citep{auer2003nonstochastic}, where the evaluators are \emph{experts} and evaluations $\{f_{i,t,j}\}$ are the experts feedback. In this case, it is possible to derive algorithms with sublinear regret w.r.t.\ the best expert in hindsight~\citep{beygelzimer2011contextual}. While this is a very general model, where no assumption is imposed either on the rewards or on the experts feedback (they could even be generated adversarially), algorithms designed for this setting tend to be over conservative in practice, as they have to be robust to any sort of data process. Furthermore, none of the evaluators may be very accurate (e.g., they all have very large variance) and targeting the performance of the best among them may not correspond to a satisfactory performance. 

Alternatively, we can frame our problem as a contextual MAB problem~\citep{agrawal2014thompson, agarwal2014taming}. We could aggregate all $J$ evaluations for arm $i$ into a context  representation
\begin{align}\label{eq:features}
    \phi_{i,t} = (f_{i,t,1}, \ldots, f_{i,t,j}, \ldots, f_{i,t,J}) \in \R^{J}.
\end{align}
Unfortunately, there are two major issues using this model: \textbf{1)} in general, the reward $r_{i,t}$ may not be a simple function (e.g., linear) of $\phi_{i,t}$; \textbf{2)} the contextual features may be noisy realizations of some ``true'' features (e.g., due to the noise factor $\epsilon_{i,t,j}$ in Eq.~\ref{eq:gen.model}). In order to deal with the first issue, we could rely on Asm.~\ref{asm:affine.setting}, which would lead to a linear contextual problem. The second issue could be dealt by using the approach proposed by~\citet{yun2017contextual} for linear contextual bandit with noisy features. While the setting in~\citep{yun2017contextual} bears some similarities (e.g., the noise distribution is assumed to be known, they consider a similar notion of relative regret and study a greedy algorithm), there remain some crucial differences: \textbf{1)} they consider unbiased features, which corresponds to a very specific instance of Asm.~\ref{asm:affine.setting} with $\alpha_j=1$; \textbf{2)} they provide guarantees only for Gaussian noise, while the algorithm designed to handle the general case has no regret guarantee.

Finally, alternative models of contextual bandit with non-deterministic features considered the case where the full distribution of the features is known~\citep{yang2020multi-feedback} or part of the features are corrupted \citep{gajane2016corrupt,bouneffouf2021corrupted}. These settings do not match the use cases studied in this paper.


\section{GENERALIZED LINEAR CASE}\label{sec:general.case}

We consider the case where the evaluator functions satisfy the generalized linear model in Asm.~\ref{asm:general.setting}.

\subsection{The Oracle Strategy}

We first define an \emph{oracle} strategy that, beside $\sigma_j$ and $g$, has prior knowledge about the parameters $\alpha_j$. At each round, the oracle receives as input the evaluations $\{f_{i,t,j}\}$ and has to select $K$ arms. We focus on oracle strategies $\mathfrak{O}$ of the following form
\begin{enumerate}[noitemsep,topsep=0pt,parsep=0pt,partopsep=0pt,leftmargin=.4cm]
    \item The oracle $\mathfrak{O}$ first aggregates the evaluations into a reward estimation $\widehat{r}^{\mathfrak{O}}_{i,t}$ for each arm $i$ using a weighted average scheme. Let $\phi_{i,t}\in\R^J$ the vector collecting all evaluations as in Eq.~\ref{eq:features} and $w\in\R^J$ a weight vector, then we define\footnote{In the linear case for $\alpha_j=1$ (i.e., the link function $g$ is the identity function) and Gaussian noise, \citep{yun2017contextual} showed the exact posterior over $r_{i,t}$ given the evaluations $\{f_{i,t,j}\}$ takes a weighted average form as in Eq.~\ref{eq:reward.estimate}.}
    \begin{equation}\label{eq:reward.estimate}
	\widehat{r}^{\mathfrak{O}}_{i,t} = \langle w , g^{-1}(\phi_{i,t})\rangle, 
    \end{equation}
    where $g^{-1}$ is the inverse of the link function applied component-wise to $\phi_{i,t}$. The choice of the weights is fixed and independent from the actual evaluations, but it may depend on the evaluation function and the noise distribution.
    \item The oracle $\mathfrak{O}$ then returns the top-K arms according to the estimates $\widehat{r}^{\mathfrak{O}}_{i,t}$, i.e.,
    \begin{equation}
	\begin{aligned}
		\gA_{t}^{\mathfrak{O}} := \arg\max_{i}^{K} \langle w, g^{-1}(\phi_{i,t}) \rangle
	\end{aligned}
\end{equation}
\end{enumerate}

The crucial aspect is then to find the weighting scheme $w$ that guarantees the best performance for the oracle. Let $i_1^\star, \ldots, i_K^\star$ be the true top-K arms and $\gA_t^{\mathfrak{O}} = \{i_1^\mathfrak{O}, \ldots, i_K^\mathfrak{O}\}$ be the \emph{estimated top-K arms} according to the estimated rewards $\widehat{r}_{i,t}^{\mathfrak{O}}$. Ideally, at each round $t$, we would like to find the oracle weights that minimize the suboptimality gap
\begin{equation}\label{eq:error_oracle}
	\Delta_t^{\mathfrak{O}} = \begin{aligned}
		\sum_{l=1}^{K} r_{i_{l}^\star,t} - \sum_{l=1}^{K} r_{i_{l}^\mathfrak{O},t}.
	\end{aligned}
\end{equation}
%
Since the rewards $r_{i,t}$ as well as the evaluations $\{f_{i,t,j}\}$ are random, it is not possible to minimize the previous expression for any possible realization using a fixed set of weights. Thus, we rather focus on minimizing a high-probability upper-bound of Eq.~\ref{eq:error_oracle}.

\begin{lemma}\label{lem:error_oracle}
Under Asm.~\ref{asm:noise} and~\ref{asm:general.setting}, with $\alpha_j$ being the parameter of the generalized linear model for each evaluator $j=1,\ldots,J$ and $\sigma_j$ being the sub-Gaussian parameter of the noise $\epsilon_{i,t,j}$, let $\delta\in(0,1)$ be a desired confidence level, then the oracle strategy designed to minimize a $(1-\delta)$-upper bound of Eq.~\ref{eq:error_oracle} is characterized by the weights solving the optimization problem
\begin{equation}\label{eq:optimal.weights.problem.glm}
    \begin{aligned}
    &\min_{w \in \R^{J}} 2\sqrt{K^3 \sum_{j=1}^{J} (w_{j}\sigma_{j})^{2}\ell_\delta} + K\sqrt{J\sum_{j=1}^{J} (w_{j} \sigma_{j})^{2}}\\
    &\textup{s.t.} \enspace\sum_{j=1}^J w_j\alpha_j = 1
    \end{aligned},
\end{equation}
where $\ell_\delta = \ln\left(\frac{K_{\text{max}}}{\delta}\right)$.
The previous problem has a closed-form solution $w^{+}\in\R^J$ such that
\begin{align}\label{eq:optimal.weights.glm}
  w_{j}^{+} = \frac{\alpha_{j}}{\sigma_{j}^{2}\|\alpha\cdot\sigma^{-1}\|}.
\end{align}
%
The resulting oracle has suboptimality gap w.p. $1-\delta$
\begin{equation}\label{eq:suboptimality.gap.glm}
    \Delta_t^+ \leq \frac{2K\sqrt{\ln\left( \frac{K_{\text{max}}e}{\delta}\right)} + K\sqrt{J}}{\|\alpha\cdot\sigma^{-1}\|}.
\end{equation}
\end{lemma}

We first remark that the previous lemma does not use Asm.~\ref{asm:rewards} and it holds for any realization of the rewards, where the probability $(1-\delta)$ is w.r.t.\ to noise in the evaluations.
We notice that the optimal weight $w_j^+$ is proportional to the ratio $\alpha_{j}/\sigma_{j}^2$, which describes the amount of ``signal'' with respect to the noise for evaluator $j$. Indeed, the oracle gives less weight to evaluators that are noisy (large $\sigma_j$), while relying more on evaluators with strong ``signal'' (large $\alpha_j$). Indeed, as $\alpha_j$ increases w.r.t.\ $\sigma_j$, the evaluations tend to be near deterministic and thus more reliable. Interestingly, the suboptimal gap in Eq.~\ref{eq:suboptimality.gap.glm} shows that the oracle improves as the number of evaluators increases (the term $\|\alpha\cdot\sigma^{-1}\|$ grows as $\sqrt{J}$) but $\Delta_{t}^{\mathfrak{O}}$ does not tend to zero even when $J\rightarrow \infty$. While this might be counterintuitive (the learner is provided with an infinite number of independent evaluations), the residual gap is due to the nonlinear nature of the generalized linear model, where the zero-mean noise added to the evaluations may be amplified or decreased through $g^{-1}$ while reconstructing the unknown parameters $\alpha_j$.

Based on the previous lemma, the oracle strategy for the generalized linear case is defined by the weights $w^+$ and we denote by $\wh{r}_{i,t}^+$ and $\gA_{t}^{+}$
the associated reward estimates and top-K arm selection rule.

\subsection{The \textsc{GLM-$\varepsilon$-greedy} Algorithm}\label{ssec:general.algorithm}

\begin{algorithm}[t]
    \small
    \DontPrintSemicolon
    \caption{\textsc{GLM-$\varepsilon$-greedy} algorithm}
    \label{alg:eps.greedy.glm}
    \KwIn{Noise parameters $\{\sigma_j\}_{j\leq J}$, confidence level $\delta$}
    \textbf{Parameters:} exploration level $\varepsilon$; number of arms to pull $K$; regularization $\lambda$\;
    Set $\mathcal{H}_{0} = \emptyset$, $\wh{\alpha} = 0$ and $w_{0} = 0$\;
    \For{$t = 1, \dots, T$}{
        Sample $Z_{t} \sim \text{Ber}\left( \varepsilon\right)$\;
        Observe evaluations for each arm $(\phi_{i,t})_{i\leq K_{t}}$\;
        \If{$Z_{t} = 1$}
        {
            Pull arms in $\gA_t$ obtained by sampling $K$ arms uniformly in $\{1, \ldots, K_{t} \}$\;
            Observe rewards $r_{i,t}$ for all $i\in \gA_{t}$\;
            Add sample to dataset $\mathcal{H}_{t} = \mathcal{H}_{t-1} \cup \big(\cup_{i\in \gA_{t}}\{ (\phi_{i,t}, r_{i,t})\}\big)$\;
            Update estimators $\widehat{\alpha}_{j,t}$ by solving
            \begin{equation}
             	\sum_{\phi,r \in \mathcal{H}_{t}} r(g(\wh{\alpha}_{t,j} \cdot r) - \phi_{j}) - \lambda \wh{\alpha}_{t,j} = 0
            \end{equation}
            Update weights $w_{t,j} = \wh{\alpha}_{t,j}/\|\wh{\alpha}_t\cdot\sigma^{-1}\|$
        }
        \Else{
            Select $\gA_{t} = \arg\max_{i}^{K}\langle w_{t}, \phi_{i,t}\rangle$\;
        }
    }
\end{algorithm}

Building on the oracle strategy defined in the previous
section, we now consider the learning problem when
the link function $g$ and the noise distribution are
known, but the learner has no knowledge of the
parameters $\alpha_j$. As customary in MAB problems,
at each round $t$, the learner observes only the
rewards of the selected  arms in $\gA_t$. The main
challenge in this setting is that a learner leveraging the evaluations $\{f_{i,t,j}\}$ is directly affected in its choices (i.e., the set $\gA_t$) by the noise $\epsilon_{i,t,j}$ generated at the beginning of round $t$. This is radically different from the standard MAB setting, where the noise (in the reward) \emph{follows} the arms played by the learner, which are then independent from any noise conditionally on the past (see App.~\ref{app:noise.correlation}). A way to circumvent this dependency is to rely on an $\varepsilon$-greedy
   strategy, where only the samples obtained in
    exploratory steps are actually used to build
	 an estimator of the unknown parameters $\alpha_j$.\footnote{While we study an $\varepsilon$-greedy type of algorithm, any type of exploration method decorrelating the estimation procedure and the exploration, like an \textit{Explore-Then-Commit} strategy, would achieve a similar regret.}
	 The resulting algorithm is detailed in
	 Alg.~\ref{alg:eps.greedy.glm}.
The core of Alg.~\ref{alg:eps.greedy.glm} is a \emph{maximum likelihood estimation} step where the algorithm learn the
parameter $\alpha$. Coherently with the evaluation model in Eq.~\ref{eq:gen.model}, the true rewards $(r_{i,t})_{i,t}$ serves as the input for the function $f_j$ (in this case the GLM model), while the evaluations $(\phi_{i,t})_{i,t}$ work as the values we need to fit.


We now compare the performance of Alg.~\ref{alg:eps.greedy.glm} to the oracle strategy and define the notion of \emph{relative} regret
\begin{align}\label{eq:regret.definition}
R_{T} = \sum_{t=1}^{T} \Big( \sum_{i\in \gA_{t}^{+}} \widehat{r}^+_{i,t} - \sum_{i\in \gA_{t}} \widehat{r}^+_{i,t} \Big),
\end{align}
where $w^+\in\R^J$ is the oracle weight vector, $\widehat{r}^+_{i,t}$ are the associated reward estimates, and $\gA_t^+$ is the oracle set of top-K arms. This notion of regret, first introduced by~\citet{yun2017contextual} in noisy contextual bandit, is comparing the quality of the arms returned by the algorithm and the oracle according to the \emph{estimated rewards}, for which the oracle is optimal (see App.~\ref{app:regret} for further discussion). The following theorem shows that if  $\varepsilon$ is properly tuned, $\varepsilon$-greedy has sublinear regret. 

\begin{theorem}\label{thm:eps.greedy.glm}
	For any $\delta\in(0,1)$, $T\geq 8$, $\lambda = J^{-1}$ and set $\varepsilon = T^{-1/3}$, let
	$\eta_{\nu, j}^{2} =  \mathbb{E}_{r\sim \nu}(g(\alpha_{j}r)^{2})$ and $\eta_{\nu,\min} = \min_{j} \eta_{\nu, j}$. Then under Asm.~\ref{asm:rewards},~\ref{asm:noise}, and~\ref{asm:general.setting}
	we have that with probability at least $1 - \delta$ the regret of Alg.~\ref{alg:eps.greedy.glm}
	is bounded as
	{\small\begin{equation*}
		\begin{aligned}
			R_{T} \leq  \wt{\mathcal{O}}\Bigg(T^{2/3}\Bigg(\frac{(1+\sqrt{J^{-1}}\|\alpha\|)\Phi S\|\sigma\|_{\infty}}{\sqrt{K}\eta_{\nu,\min}} + \frac{\|g\|_{\infty}^{3}}{K^{2}\eta_{\nu, \min}^{2}}\Bigg)\Bigg) 
		\end{aligned}
	\end{equation*}}
	where $\|g\|_{\infty} = \max_{j, x\in [0, C]} g(\alpha_{j} x)$ and
	{\small\begin{equation*}
		\begin{aligned}
		\Phi &=  2K\|\sigma\|_{\infty}\left( 2\sqrt{J} + \sqrt{K\ln\left(\frac{eK_{\max}}{K\delta}\right)}\right) + K\|g \|_{\infty} \\
		S &= \left( \| \alpha\cdot \sigma^{-2}\|_{2} + \|\sigma^{-2}\|_{2}\right)\frac{\|\sigma^{-2}\cdot\alpha\|_{2}}{\| \sigma^{-1}\cdot\alpha\|_{2}^{4}} + \left(\frac{\| \sigma^{-1}\|_{4}}{\| \sigma^{-1}\cdot\alpha\|_{2}}\right)^{2}.
	\end{aligned}
	\end{equation*}}
\end{theorem}

As expected, the regret of GLM-$\varepsilon$-greedy increases as $\widetilde{O}(T^{2/3})$. While this shows that the algorithm is able to approach the performance of the oracle, it also illustrates the difficulty of this setting, where the strict decoupling between explorative and exploitative steps used to guarantee a consistent estimation process translates into a higher regret. Interestingly, this result matches the regret of~\citet{yun2017contextual} for contextual bandit with noisy features.

In order to investigate the main terms appearing in the previous bound, we consider the case where all evaluators share the same parameters $\alpha_{j} = \alpha_{0}$ and $\sigma_{j} =\sigma_{0}$ for some $(\alpha_{0}, \sigma_{0})\in \mathbb{R}_{+}^{2}$ (with $\alpha_{0} \geq 1$). In this case, the regret bound of Thm.~\ref{thm:eps.greedy.glm} reduces to
\begin{equation*}
	\begin{aligned}
	R_T \leq \wt{O}\Bigg( T^{2/3} \sqrt{K}\left( 1 + \sqrt{\frac{K}{J}}\right)\frac{\sigma_{0}^{2}}{\alpha_{0}}\Bigg)
	\end{aligned}
\end{equation*}
We first notice that the bound scales as $\max\{\sqrt{K}, K\sqrt{J^{-1}}\}$. Similar to the suboptimality gap for the oracle (see Lemma~\ref{lem:error_oracle}), when $J\rightarrow \infty$, the bound improves but does not decrease down to $0$ and rather converges to $\sqrt{K}$. The dependency on $K$ may be surprising, since for $K=K_{\max}$ the regret is trivially 0 at each round (i.e., the algorithm returns all arms and it cannot make any error in the ranking). However, this dependency comes from the fact that in the regret analysis we bound the norm of the evaluations for the selected jobs, which scales linearly with $K$. We believe a more refined analysis could alleviate this dependency. Last, the regret depends inversely on the "signal-to-noise" ratio $\frac{\alpha_0}{\sigma_{0}^{2}}$.



\section{LINEAR CASE}\label{sec:affine}

We consider the special case of linear evaluations and show how this relatively minor change to the problem has a major impact on how to approach the learning problem and the regret.

\subsection{The Linear Oracle Strategy}
Similar to the GLM case, we define the oracle strategy that defines an estimated reward $\wh{r}_{i,t}^{\mathfrak{O}} = \langle w, \phi_{i,t} \rangle$. Then the oracle ranks arms according to $\wh{r}_{i,t}^{\mathfrak{O}}$ and selects the set of top-K $\gA_t^\mathfrak{O}$ accordingly. We then optimize weights to minimize a high-probability upper bound to the suboptimality gap $\Delta_{t}^{\mathfrak{O}}$.


\begin{lemma}\label{lem:error_oracle_affine}
	Under Asm.~\ref{asm:noise} and~\ref{asm:affine.setting}, with $\alpha_j$, being the parameter for each evaluator $j=1,\ldots,J$, $\sigma_j$ being the sub-Gaussian parameter of the noise $\epsilon_{i,t,j}$, let $\delta\in(0,1)$ be a desired confidence level, then the oracle strategy designed to minimize a $(1-\delta)$-upper bound of Eq.~\ref{eq:error_oracle} is characterized by the weights obtained as the solution of the optimization problem
	{\small\begin{equation}\label{eq:optimal.weights.problem.glm}
		\begin{aligned}
		&\min_{w \in \R^{J}} 2\sqrt{K^3 \sum_{j=1}^{J} (w_{j}\sigma_{j})^{2}\ell_\delta}\textup{ s.t.} \enspace\sum_{j=1}^J w_j\alpha_j = 1
		\end{aligned},
	\end{equation}}
	where $\ell_\delta = \ln\left(\frac{K_{\text{max}}}{\delta}\right)$.
	The previous problem has a closed-form solution $w^{+}\in\R^J$ such that
	\begin{align}\label{eq:optimal.weights.glm}
	  w_{j}^{+} = \frac{\alpha_{j}}{\sigma_{j}^{2}\|\alpha\cdot\sigma^{-1}\|^{2}}.
	\end{align}
The resulting oracle has suboptimality gap w.p. $1-\delta$
	\begin{equation}
		\Delta_t^+ \leq \frac{2K\sqrt{\ln\left( \frac{K_{\text{max}}e}{\delta}\right)}}{\|\alpha\cdot\sigma^{-1}\|}
	\end{equation}
\end{lemma}

The weights of the oracle have the same expression in the GLM case in Lemma~\ref{lem:error_oracle}. Nonetheless, the suboptimality gap is smaller than in Lemma~\ref{lem:error_oracle} as the linear structure allows to concentrate the noise of the evaluations, unlike in the GLM setting where the potential non linearity of $g$ forces us to study the worst-case scenario. Notably, in the linear case, we see that as $J$ tends to infinity the suboptimality gap tends to zero.


\subsection{Evaluation-Structure-Aware Greedy}

Similarly to Sec.~\ref{ssec:general.algorithm}, the learner has no knowledge of the parameters $\alpha_j$ but only knows the noise distribution and the linear structure of the evaluations. The main estimation difficulty of the general case still applies to this setting, i.e., using samples obtained by selecting arms based on the evaluations may introduce a bias in the estimation process.
%

Instead of introducing explicit exploration steps to gather ``unbiased'' samples, as done in the GLM case, we exploit a more subtle property for this case. We notice for any evaluator $j$, the expected evaluation is
\begin{align}\label{eq:expected.evaluation}
    \mathbb{E}\big[f_{i,t,j}\big] = \mathbb{E}\big[\alpha_j r_{i,t} + \epsilon_{i,t,j}\big] = \alpha_j \wb{r},
\end{align}
where $\wb{r} = \mathbb{E}[r_{i,t}]$ is the expectation of the reward distribution $\nu$ in Asm.~\ref{asm:rewards}. Consider an oracle strategy that is fed with the parameters $\alpha_j \wb{r}$, then the optimal weights in Eq.~\ref{eq:optimal.weights.linear} become $\wt{w}^+ = w^+/\wb{r}$.
While this leads to estimates $\wt{r}^+_{i,t}$ that are biased w.r.t.\ the oracle estimates $\wh{r}^+_{i,t}$, the factor $1/\wb{r}$ is evaluator- and arm-independent and it does not impact the ranking returned by this \emph{biased} oracle, i.e., $\wt{\gA}_t = \gA^+_t$.
This in striking contrast with the GLM where it is not possible to easily evaluate a vector proportional to $\alpha$ because of the potential non-linear behavior of $g$.

Building on this evidence, we define an algorithm, the \emph{Evaluation-Structure-Aware Greedy} (\textsc{ESAG}) in Alg.~\ref{alg:greedy.affine}, that avoids the use of the observed rewards altogether, thus removing the statistical dependency between noise and decisions, and rather tries to estimate the expected evaluations in Eq.~\ref{eq:expected.evaluation} (see Eq.~\ref{eq:estimate.alpha.affine} in Alg.~\ref{alg:greedy.affine}) and use them to build estimates, as in the biased oracle. Since \textsc{ESAG} only relies on the evaluations available at round $t$, it does not need any \emph{explicit} exploration strategy to collect useful information, and it executes greedy actions according to the current weights $w_{t}$ at each round. We can derive the following regret guarantees.\footnote{While we derive Thm.~\ref{thm:greedy.linear} for a greedy algorithm, similar results hold for an optimistic exploration strategy.}

\begin{algorithm}[t]
    \small
    \DontPrintSemicolon
    \caption{The \emph{Evaluation-Structure-Aware Greedy} (\textsc{ESAG}) algorithm for the linear case.}
    \label{alg:greedy.affine}
    \KwIn{Noise parameters $\{\sigma_j\}_{j\leq J}$, confidence level $\delta$}
    \textbf{Parameters:} number of arms to pull $K$\;
    Set $\wh{\alpha} = 0$, $w_{0} = 0$ and $N_{t} = 0$\;
    \For{$t = 1, \dots, T$}{
        Observe evaluations for each arm $(\phi_{i,t})_{i\leq K_{t}}$\;
		Select $\gA_{t}= \arg\max_{i}^{K} \langle w_{t}, \phi_{i,t}\rangle$\;
        Observe rewards $r_{i,t}$ for all $i\in \gA_{t}$\;
		Update $\wh{\alpha}_{t}$ as:
		\begin{equation}\label{eq:estimate.alpha.affine}
			\wh{\alpha}_{t+1} = \wh{\alpha}_{t}\frac{N_{t} + K_{t} - N_{t}K_{t}}{(N_{t} + K_{t})N_{t}} + \frac{\sum_{i=1}^{K_{t}} \phi_{i,t}}{K_{t}}
		\end{equation}
		Update $N_{t+1} = N_{t} + K_{t}$ and $w_{t}$ as
		\begin{align}\label{eq:optimal.weights.linear}
		w_{t+1, j} = \frac{\wh{\alpha}_{t+1, j}}{\sigma_{j}^{2}\left\| \wh{\alpha}_{t+1,j}\cdot\sigma^{-1}\right\|^{2}}
		\end{align}

    }
\end{algorithm}

\begin{theorem}\label{thm:greedy.linear}
	Under Asm.~\ref{asm:rewards},~\ref{asm:noise}, and~\ref{asm:affine.setting} for any $\delta\in(0,1)$, $T\geq 1$ with probability at least $1 - \delta$ the regret of Alg.~\ref{alg:greedy.affine}
	is bounded by:
	{\small\begin{equation*}
		\begin{aligned}
			R_{T}\leq \mathbb{E}_{r\sim\nu}(r)\Phi' K\Bigg[ \left( \frac{\Phi'}{\mathbb{E}_{r\sim\nu}(r)\|\alpha\|_{\infty}\sqrt{K}}\right)^{2} + \frac{8\sqrt{T}\Phi' S}{\sqrt{\bar{K}_{T}}}\Bigg]
		\end{aligned}
	\end{equation*}}
where $\bar{K}_T = \frac{T}{\sum_{t=1}^{T} \frac{t-1}{\sum_{l=1}^{t-1} K_{l}}} \geq K$ is the harmonic average of the number of arms over $T$ steps, $S$ is the same as in Thm.~\ref{thm:eps.greedy.glm} and
{\small
\begin{equation*}
	\begin{aligned}
	\Phi' = 2\|\alpha\|C\ln\left(\frac{4}{\delta}\right) + 2\|\sigma\|_{\infty}\left( 2\sqrt{J} + \sqrt{K\ln\left(\frac{K_{\text{max}}}{K\delta}\right)}\right)
	\end{aligned}
\end{equation*}}
\end{theorem}

The most interesting aspect of the previous theorem is that the regret is of order $\wt{O}(\sqrt{T})$, since \textsc{ESAG} does not pay for the sharp separation between exploration and exploitation steps as for \textsc{GLM-$\varepsilon$-greedy}.


\section{EMPIRICAL VALIDATION}\label{sec:experiments}

In order to study different aspects of the settings and algorithms introduced in the previous sections, we focus on both synthetic and real data experiments. We report further results in the supplementary material.

\subsection{Synthetic Data}

We first validate our algorithms on synthetic data. We consider $K_t = K_{\max} = 20$ arms where for each arm we get $J=10$ evaluations. The reward distribution $\nu$ is a Gaussian distribution centered at $0$ and truncated between $[0,20]$. We consider both the logistic case with $g(x) = (1 + \exp(-x))^{-1}$ and the linear case where for all evaluators $f_{j}(x) = \alpha_{j}x$.
At the beginning of each experiment, we draw the coefficients $\alpha_j$ and the parameters $\sigma_j$ from uniform distributions in $[\alpha_0/2; 3\alpha_0/2]$ and $[\sigma_0/2; 3\sigma_0/2]$ respectively. As discussed in Sect.~\ref{sec:general.case}, a critical term characterizing the problem structure is the ratio $\alpha_j / \sigma_j$. We then set $\alpha_0=1$ and consider three different values for $\sigma_0$ such that $\alpha_0/\sigma_0 \in \{0.1, 1, 10\}$. Finally, the noise in the evaluations are generated as $\eps_{i,t,j} \sim\mathcal{N}(0, \sigma_j^2)$ in the linear case whereas in the GLM case $\eps_{i,t,j}$ is drawn from a truncated centered gaussian distribution with variance $2\sigma_{j}^{2}$. We average all results over $80$ runs and we report $95\%$ confidence intervals. Additional details are reported in the supplement.



\textbf{Oracle performance.} Before investigating the performance of the learning algorithms, we compare the performance of the oracle strategy to a simple average strategy that defines $\wh{r}_{i,t}^{\text{avg}} = 1/J \sum_{j=1}^J f_{i,t,j}$ and ranks arms accordingly. We also study how the suboptimality gap of the oracle changes as $J$ increases for both the GLM and linear case. As shown in Fig.~\ref{fig:single_plot}a, in both settings that oracle strategies outperform the simple average. Furthermore, as predicted by Lemma~\ref{lem:error_oracle} and~\ref{lem:error_oracle_affine}, the oracle suboptimality gap decreases as $J$ increases, but it plateaus to a fixed value for GLM, while it tends to zero for the linear case.

\subsubsection{The GLM Case}

We consider different types of algorithms:
\textsc{GLM-$\varepsilon$-greedy} (Alg.~\ref{alg:eps.greedy.glm}); \textsc{GLM-$\varepsilon$-greedy-all}, it has same structure as Alg.~\ref{alg:eps.greedy.glm} but uses samples from both explorative and exploitative steps to build the estimator $\wh\alpha$; \textsc{GLM-EvalBasedUCB}, it uses all samples to build an estimator $\wh\alpha$ and leverages a high-probability confidence interval on $\wh\alpha$ to derive optimistic weights $w$ and rank and select arms accordingly;
\textsc{GLM-LinUCB}, the algorithm of~\citet{abbasi2011improved} using $g^{-1}(\phi_{i,t})$ as features;
\textsc{GLM-ESAG}, Alg.~\ref{alg:greedy.affine} adapted for the GLM  case;
\textsc{Rand}, the fully random strategy;
\textsc{GLM-Greedy}, the greedy strategy using the MLE estimator $\wh\alpha$;
\textsc{Exp4.P}, the bandit with expert advice algorithm in~\cite{beygelzimer2011contextual}.


\textbf{Estimation Bias.} As discussed in Sect.~\ref{ssec:general.algorithm}, one of the critical aspects that motivated the use of an $\varepsilon$-greedy approach and led to the $\widetilde{O}(T^{2/3})$ is the fact that whenever the set of arms is chosen according to the noisy evaluations $\{f_{i,t,j}\}$, the dependency between $\epsilon_{i,t,j}$ and $\gA_t$ may create a bias when estimating the parameters $\alpha_j$ using the samples $r_{i,t}$ observed after selecting $\gA_t$. We illustrate this effect in Fig.~\ref{fig:single_plot}b, where we report the error of the estimates $\wh{\alpha}_t$ computed by \textsc{GLM-$\varepsilon$-greedy} and \textsc{GLM-$\varepsilon$-greedy-all} w.r.t.\ the true parameters $\alpha$. While the error of the estimator computed by  \textsc{GLM-$\varepsilon$-greedy} decreases over time, the error for \textsc{GLM-$\varepsilon$-greedy-all} has a residual bias due to the estimation procedure. Similar results can be shown for all the algorithms (e.g., \textsc{GLM-ESAG}) that rely on samples generated by selecting $\gA_t$ based on the evaluations $\{f_{i,t,j}\}$ and they can be reproduced in the linear setting as well. 


\textbf{Regret Performance.} We compare the performance of different learning algorithms in terms of relative regret w.r.t.\ the oracle defined in Sec.~\ref{sec:general.case}.
%
%
We remove from Fig.~\ref{fig:single_plot}c all algorithms (i.e., \textsc{RAND}, \textsc{Greedy}, \textsc{EXP4.P}) that suffer large linear regret to avoid loosing resolution on the other algorithms. While \textsc{GLM-ESAG} and \textsc{GLM-LinUCB} have better regret, they both have a linear regret that keeps increasing over time, while \textsc{GLM-$\varepsilon$-greedy} has a sublinear regret.\footnote{Notice that we have not actively tried to find problem parameters that would make the linear regret of \textsc{GLM-ESAG} and \textsc{GLM-LinUCB} larger than \textsc{GLM-$\varepsilon$-greedy} in a shorter time. The linear regret of \textsc{GLM-ESAG} is due to the residual estimation bias illustrated in Fig.~\ref{fig:single_plot}b.}





\begin{figure*}
    \centering
    \includegraphics[width=\textwidth]{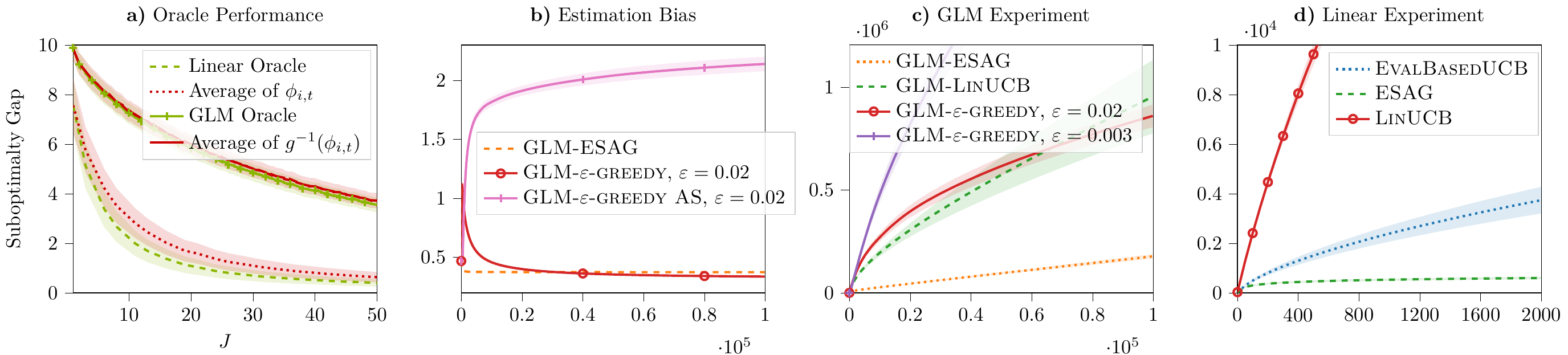}
    \caption{\small 
    \textbf{a)} Average suboptimality gap $\Delta_t^\mathfrak{O}$ for the oracle strategy in the GLM and linear cases as a function of the number of evaluators $J$.
    \textbf{b)} Estimation error $\| \wh{\alpha}_{t} - \alpha\|$.
    \textbf{c)} Regret w.r.t.\ to the oracle as defined in Sect.~\ref{sec:general.case} for the GLM case.
    \textbf{d)} Regret w.r.t.\ to the oracle as defined in Sect.~\ref{sec:general.case} for the linear case in the high noise regime (i.e., $\alpha_{j} = 0.1\sigma_{j}$).
    }
    \label{fig:single_plot}
\end{figure*}

\subsubsection{The Linear Case}


We also study the linear case and compare the linear versions of the algorithms described above. As illustrated in Fig.~\ref{fig:single_plot}d, the crucial difference w.r.t.\ the GLM case is that now \textsc{ESAG} has sublinear regret, whereas other algorithms have linear regret. Interestingly, \textsc{EvalBasedUCB}, despite the bias introduced by using samples obtained by selecting actions based on the noisy observation, is able to learn a good strategy and it outperforms \textsc{LinUCB}, but still suffers linear regret.

\subsection{Content Review Prioritization}

We now move to a real-world problem to investigate the performance of our algorithms when their assumptions are no longer verified.

\textbf{Data description.} We consider a small dataset of content shared on a large social media firm that has been reported for violating the platforms' community standards. In order to ensure that the most harmful content is prioritized for reviewing, the platform assigns badness score for each piece of content which increases with the severity of the content. We consider four different classifiers that provide badness estimates for the sampled content. Each of these classifiers are trained on the real data and follow different modeling architectures to predict the badness score. Our objective is to leverage scores from these four classifiers to identify the most harmful subset of content and flag them for the platform to review and action them.

\begin{table}\label{fig:cumul.badness}
\small
    \caption{Cumulative badness over $T = 2000$ steps\vspace{-0.1in}}
    \label{fig:cumul.badness}
    \begin{center}
        \begin{tabular}{ |c|c| }
        \hline
        Alg. & Badness 
        \\
        \hline
        \hline
        \textsc{Rand} & $30076.18$ 
        \\
        \textsc{GLM-$\varepsilon$-greedy} & $53403.6$ 
        \\
        \textsc{GLM-$\varepsilon$-greedy-all} & $53419$ 
        \\
        \textsc{GLM-EvalBasedUCB} & $53328.3$
        \\
        \textsc{GLM-LinUCB} & $30910.5$ 
        \\
        \textsc{GLM-ESAG} & $53393.9$
        \\
        \textsc{GLM-Greedy} & $53371.9$  
        \\
        \textsc{Exp4.P} & $47766.7$ 
        \\
        \textsc{EvalBasedUCB} & $\mathbf{90768.2}$ 
        \\
        \textsc{LinUCB} & $\mathbf{90332.2}$ 
        \\
        \textsc{ESAG} &  $\mathbf{90790.5}$ 
        \\
        \textsc{Greedy} & $69956.3$ 
        \\
        \hline
        \end{tabular}
        \vspace{-0.2in}
    \end{center}
\end{table}

\textbf{Performance.} We evaluate the cumulative badness of the content selected by the algorithms. The higher the score the better. The results are reported in Tab.~\ref{fig:cumul.badness}.
We first notice that all learning algorithms perform significantly better than the random strategy, thus indicating that the GLM and linear assumptions are accurate enough to return meaningful rankings. Nonetheless, we notice that GLM-based algorithms do not perform as well as linear ones, probably due to the choice of the logistic function, which, in this case, does not fit data accurately. On the other hand, \textsc{ESAG} is the algorithm that performs best, followed by \textsc{EvalBasedUCB} and \textsc{LinUCB}. Notice that in the case of real data, we may even expect \textsc{LinUCB} to perform better, since it relies on somewhat less restrictive assumptions (no assumption is made on how the features are generated) and it relies on the true rewards to estimate parameters (this is also the case for \textsc{EvalBasedUCB}). This is in contrast with \textsc{ESAG} that exclusively builds on the evaluations, which clearly do not respect an exact linear model, to estimate the unknown parameters. This shows that, even in problems where the assumptions do not hold, \textsc{ESAG} is robust enough and it is competitive w.r.t.\ a large variety of algorithms.


\section{DISCUSSION}\label{sec:discussion}

We studied a MAB problem where the learner is provided with noisy and biased evaluations of the true reward for each arm. We showed that under specific assumptions it is possible to design learning algorithms that are able to compete to oracle strategies both in theory and in practice. The empirical validation on real data also shows that this model and the associated algorithms are promising for solving challenging real-world problems.

\textbf{Extensions.} There is a number of directions that could be pursued to extend our current results.

\begin{itemize}[noitemsep,topsep=0pt,parsep=0pt,partopsep=0pt,leftmargin=.4cm]
    \item \textit{Evaluation functions.} The GLM and linear assumptions are relatively strong. A natural venue of improvement is to generalize our results to richer function spaces, such as Gaussian processes.
    \item \textit{Persistent arms.} While we assume that the set of arms is ``refreshed'' at each round, we can easily extend our setting to the case where all the arms that are not selected in $\gA_t$ remain in the pool of arms available at the next round (e.g., content that has not been reviewed). All our results naturally extend to this case, except for \textsc{ESAG}, which would not be able to use $K_t$ samples at each round, but would rather get $K$ new samples corresponding to the arms replaced at each round.
    \item \textit{Heteroschedastic noise.} In the model illustrated in Fig.~\ref{foot:evaluation.function}, the noise $\epsilon_{i,t,j}$ is heteroschedastic, where the variance may depend on the reward value $r$. While our model leverages an upper-bound $\sigma_j$ on the actual variance, better adapting to the reward-dependent variance may improve the performance.
\end{itemize}

\bibliography{biblio.bib}

\begin{thebibliography}{}

\bibitem[Abbasi{-}Yadkori et~al., 2011]{abbasi2011improved}
Abbasi{-}Yadkori, Y., P{\'{a}}l, D., and Szepesv{\'{a}}ri, C. (2011).
\newblock Improved algorithms for linear stochastic bandits.
\newblock In {\em {NIPS}}, pages 2312--2320.

\bibitem[Agarwal et~al., 2014]{agarwal2014taming}
Agarwal, A., Hsu, D.~J., Kale, S., Langford, J., Li, L., and Schapire, R.~E.
  (2014).
\newblock Taming the monster: {A} fast and simple algorithm for contextual
  bandits.
\newblock In {\em {ICML}}, volume~32 of {\em {JMLR} Workshop and Conference
  Proceedings}, pages 1638--1646. JMLR.org.

\bibitem[Agrawal and Goyal, 2013]{agrawal2014thompson}
Agrawal, S. and Goyal, N. (2013).
\newblock Thompson sampling for contextual bandits with linear payoffs.
\newblock In {\em {ICML} {(3)}}, volume~28 of {\em {JMLR} Workshop and
  Conference Proceedings}, pages 127--135. JMLR.org.

\bibitem[Auer et~al., 2003]{auer2003nonstochastic}
Auer, P., Cesa-Bianchi, N., Freund, Y., and Schapire, R.~E. (2003).
\newblock The nonstochastic multiarmed bandit problem.
\newblock {\em SIAM J. Comput.}, 32(1):48–77.

\bibitem[Beygelzimer et~al., 2011]{beygelzimer2011contextual}
Beygelzimer, A., Langford, J., Li, L., Reyzin, L., and Schapire, R.~E. (2011).
\newblock Contextual bandit algorithms with supervised learning guarantees.
\newblock In {\em {AISTATS}}, volume~15 of {\em {JMLR} Proceedings}, pages
  19--26. JMLR.org.

\bibitem[Bouneffouf, 2021]{bouneffouf2021corrupted}
Bouneffouf, D. (2021).
\newblock Corrupted contextual bandits: Online learning with corrupted context.
\newblock In {\em {ICASSP}}, pages 3145--3149. {IEEE}.

\bibitem[D{\'e}ry et~al., 2019]{Dry2019PatientPT}
D{\'e}ry, J., Ruiz, A.~B., Routhier, F., Gagnon, M.-P., C{\^o}t{\'e}, A.,
  Ait-kadi, D., B{\'e}langer, V., Deslauriers, S., and Lamontagne, M.-{\`E}.
  (2019).
\newblock Patient prioritization tools and their effectiveness in non-emergency
  healthcare services: a systematic review protocol.
\newblock {\em Systematic Reviews}, 8.

\bibitem[Filippi et~al., 2010]{filippi2010parametric}
Filippi, S., Cappe, O., Garivier, A., and Szepesv\'{a}ri, C. (2010).
\newblock Parametric bandits: The generalized linear case.
\newblock In Lafferty, J., Williams, C., Shawe-Taylor, J., Zemel, R., and
  Culotta, A., editors, {\em Advances in Neural Information Processing
  Systems}, volume~23. Curran Associates, Inc.

\bibitem[Gajane et~al., 2016]{gajane2016corrupt}
Gajane, P., Urvoy, T., and Kaufmann, E. (2016).
\newblock {Corrupt Bandits}.
\newblock In {\em EWRL 13}.

\bibitem[Goldberg et~al., 2001]{GoldbergRGP01}
Goldberg, K.~Y., Roeder, T., Gupta, D., and Perkins, C. (2001).
\newblock Eigentaste: {A} constant time collaborative filtering algorithm.
\newblock {\em Inf. Retr.}, 4(2):133--151.

\bibitem[Li et~al., 2020]{li2020competencelevel}
Li, C., Fisher, E., Thomas, R., Pittard, S., Hertzberg, V., and Choi, J.~D.
  (2020).
\newblock Competence-level prediction and resume {\&} job description matching
  using context-aware transformer models.
\newblock In {\em {EMNLP} {(1)}}, pages 8456--8466. Association for
  Computational Linguistics.

\bibitem[Li et~al., 2017]{li2017provably}
Li, L., Lu, Y., and Zhou, D. (2017).
\newblock Provably optimal algorithms for generalized linear contextual
  bandits.
\newblock In {\em {ICML}}, volume~70 of {\em Proceedings of Machine Learning
  Research}, pages 2071--2080. {PMLR}.

\bibitem[Pedregosa et~al., 2011]{scikitlearn}
Pedregosa, F., Varoquaux, G., Gramfort, A., Michel, V., Thirion, B., Grisel,
  O., Blondel, M., Prettenhofer, P., Weiss, R., Dubourg, V., Vanderplas, J.,
  Passos, A., Cournapeau, D., Brucher, M., Perrot, M., and Duchesnay, E.
  (2011).
\newblock Scikit-learn: Machine learning in {P}ython.
\newblock {\em Journal of Machine Learning Research}, 12:2825--2830.

\bibitem[Provenzano et~al., 2020]{provenzano2020machine}
Provenzano, A.~R., Trifirò, D., Datteo, A., Giada, L., Jean, N., Riciputi, A.,
  Pera, G.~L., Spadaccino, M., Massaron, L., and Nordio, C. (2020).
\newblock Machine learning approach for credit scoring.

\bibitem[Riquelme et~al., 2018]{RiquelmeTS18}
Riquelme, C., Tucker, G., and Snoek, J. (2018).
\newblock Deep bayesian bandits showdown: An empirical comparison of bayesian
  deep networks for thompson sampling.
\newblock In {\em {ICLR} (Poster)}. OpenReview.net.

\bibitem[Yang et~al., 2020]{yang2020multi-feedback}
Yang, L., Yang, J., and Ren, S. (2020).
\newblock Multi-feedback bandit learning with probabilistic contexts.
\newblock In {\em {IJCAI}}, pages 3087--3093. ijcai.org.

\bibitem[Yun et~al., 2017]{yun2017contextual}
Yun, S., Nam, J.~H., Mo, S., and Shin, J. (2017).
\newblock Contextual multi-armed bandits under feature uncertainty.
\newblock {\em CoRR}, abs/1703.01347.

\end{thebibliography}
\bibliographystyle{apalike}


\clearpage
\begin{appendix}
\addcontentsline{toc}{section}{Appendix} 
\part{Appendix} 
\parttoc %

\section{LOWER BOUND}\label{app:lower.bound}

    In this appendix, we prove why no algorithm can compute the top $K$ arms at every time step. More, precisely we prove the following Lemma.

\begin{lemma}\label{lem:lower.bound}
Let consider the linear case in Asm.~\ref{asm:affine.setting} and parameters $K_t=K_{\max}$, $K=1$, $\alpha_j=1$ for all $j=1,\ldots,J$ and a noise distributed as $\mathcal{N}(0,\sigma_j)$ with $\sigma_j=\sigma_0$ for all $j=1,\ldots,J$. The learner $\mathfrak{A}$ receives as input the evaluations $f_{i,t,j} = r_{i,t} + \epsilon_{i,t,j}$ and we denote by $\mathfrak{A}(\{f_{i,t,j}\})$ the arm returned by $\mathfrak{A}$. {\color{black}For all arms $i\leq K_{t}$ the reward $r_{i,t}$ is sampled from a Bernoulli distribution $\text{Ber}(1/2)$. At every step, let's define $I_{t}^{\star} = \{ i\leq K_{t} \mid r_{i,t} = 1\}$ the set of optimal arms.} Then any learner $\mathfrak{A}$ has a fixed non-zero probability to return the wrong ranking at each step, i.e.,
\begin{align}
\forall t, \forall \mathfrak{A}, \exists \delta > 0, \mathbb{P}_{\text{Ber},\{\epsilon_{i,t,j}\}} \Big[ \mathfrak{A}(\{f_{i,t,j}\}) {\color{black}\not\in I_{t}^{\star}}\Big] \geq \delta.
\end{align}
\end{lemma}

\begin{proof}
    We reason by contradiction and assume that there exists a deterministic algorithm $\mathfrak{A}$ and a time step $t$ such that
    \begin{align}\label{eq:false_hyp}
        \forall \delta > 0, \qquad \mathbb{P}_{\text{Ber},\{\epsilon_{i,t,j}\}} \Big[ \mathfrak{A}(\{f_{i,t,j}\}) {\color{black}\not\in I_{t}^{\star}} \Big] \leq \delta.
    \end{align}
    {\color{black}Now given the rewards $(r_{i,t})_{i\leq K_{t}}$} the distribution of the evaluations $(f_{i,t,j})_{j}$ is distributed as $\mathcal{N}\left(({\color{black}r_{i,t}})_{j\leq J}, \sigma_{0}I_{J}\right)$ for each $i$. 
    {\color{black} Now, let's consider the set of sub optimal arms, $I_{t}^{-} := \{ i\leq K_{t} \mid r_{i,t} = 0\}$ because the rewards are independent then with probability at least $\frac{1}{4}$ the set $I_{t}^{+}$ and $I_{t}^{-}$ are both non-empty.}
    {\color{black}For an index $i_{0}\leq K_{t}$, let $X^{0,1} \in \mathbb{R}^{K_{t} \times J}$ and  $X^{1, 0}\in   \mathbb{R}^{K_{t} \times J}$} be two random variables distributed sampled
    from the same distributions as the {\color{black}$(f_{i,t,j})_{i,j}\mid I_{t}^{\star} = \{i_{0}\}, I_{t}^{-} = \{ i_{-0}\}$ and $(f_{i,t,j})_{i,j}\mid I_{t}^{-} = \{i_{0}\}, I_{t}^{\star} = \{ i_{-0}\}$ respectively}. We then have that:
    \begin{align}
        \mathbb{P}\left( {\color{black}X^{0,1} = X^{1,0}}\right) > 0
    \end{align}
    This is a direct consequence of the independence of the noise and {\color{black} rewards}. Indeed,
    \begin{align}
        \mathbb{P}\left( X^{0,1} = X^{1,0}\right) &= \mathbb{P}\left( \forall (i,j)\in \{1, \cdots, K_{t}\}\times \{1, \dots, J\} \text{ } X^{0,1}_{i,j} = X^{1,0}_{i,j}\right)\\
        &= \prod_{i,j} \mathbb{P}( X^{0,1}_{i,j} = X^{1,0}_{i,j})
    \end{align}
    But for a given couple {\color{black}$(i,j)\in \{1, \dots, K_{t}\} \times \{1, \dots, J\}$, $X^{0,1}_{i,j} \sim \mathcal{N}(\mathds{1}_{\{i = i_{0}\}} ,\sigma_{0}^{2})$, $X^{1, 0}_{i,j} \sim \mathcal{N}(\mathds{1}_{\{i \neq i_{0}\}}, \sigma_{0}^{2})$} and are independent. Therefore, the probability of those random variables being equal is:
    \begin{align}
        {\color{black}\mathbb{P}\left( X^{0,1}_{i,j}= X^{1,0}_{i,j}\right) = \frac{1}{2\pi\sigma_{0}^{2}}\int_{-\infty}^{+\infty}  \exp\left(- \frac{\left[\left(x - \mathds{1}_{\{i = i_{0}\}} \right)^{2} + \left(x - \mathds{1}_{\{i \neq i_{0}\}} \right)^{2}\right]}{2\sigma_{0}^{2}}\right)dx > 0}
    \end{align}
    Let's note $\gamma> 0$ such that $\gamma \leq {\color{black}\mathbb{P}\left( X^{0,1}_{i,j} = X^{1,0}_{i,j}\right)}$ for all $i$ then we have that:
    $\mathbb{P}\left( X^{0,1} = X^{1, 0}\right) \geq \gamma^{K_{t}J} >0$. This implies that there exists a $\delta_{0}$ such that {\color{black}$\mathbb{P}_{\text{Ber},\{\epsilon_{i,t,j}\}} \Big[ \mathfrak{A}(X^{0,1}) \in I_{t}^{\star}, X^{0,1} = X^{1,0}, I_{t}^{\star} = \{ i_{0}\}, I_{t}^{-} = i_{-0}\Big] \geq \delta_0 >0$, thanks to Eq.~(\ref{eq:false_hyp})} but then on this event we have that: 
    {\color{black}\begin{align}
        \delta_{0} \leq \mathbb{P}_{\text{Ber},\{\epsilon_{i,t,j}\}} \Big[ \mathfrak{A}(X^{0,1}) \in I_{t}^{\star}, X^{0,1} = X^{1,0}, I_{t}^{\star} = \{ i_{0}\}, I_{t}^{-} = i_{-0}\Big] &= \mathbb{P}_{\text{Ber},\{\epsilon_{i,t,j}\}} \Big[ \mathfrak{A}(X^{1,0}) = i_{0}, I_{t}^{\star} = \{ i_{0}\}, I_{t}^{-} = i_{-0}\Big]\\
        &\leq \mathbb{P}_{\text{Ber},\{\epsilon_{i,t,j}\}} \Big[ \mathfrak{A}(X^{1, 0}) =  i_{0}\Big] \\
        & \leq \mathbb{P}_{\text{Ber},\{\epsilon_{i,t,j}\}} \Big[ \mathfrak{A}(X^{1, 0}) \not\in  I_{t}^{\star}\Big] \leq \frac{\delta_{0}}{2}
    \end{align}}
    thanks to Eq.~\ref{eq:false_hyp}. This is not possible therefore by contradiction we have the result.
\end{proof}

Lem.~\ref{lem:lower.bound} shows that there exists a instance of the problem studied in this paper where it is not possible to retrieve systematically the $K$ actions maximiziing the reward $(r_{i,t})_{i\leq K_{t}})$. Therefore defining the regret with
respect to the true top $K$ actions.

\section{NOISE CORRELATION ISSUE}\label{app:noise.correlation}

The standard result on GLM are based theorems with the same structure as the following proposition \citep{filippi2010parametric}:

\begin{proposition}\label{prop:glm}
    Let $(\mathcal{F}_{k})_{k\geq 0}$ be a filtration, $(m_{k})_{k\geq 0}$ be an $\mathbb{R}^{d}$-valued stochastic process adapted to $(\mathcal{F}_{k})_{k}$,
    $(\eta_{k})_{k}$ be a real-valued martingale difference process adapted to $(\mathcal{F}_{k})_{k}$. Assume that $\eta_{k}$ is conditionnally sub-Gaussian in the sense that
    there exists some $R>0$ such that for any $\gamma \geq 0$, $k\geq 1$,
    \begin{align}
        \mathbb{E}(\exp(\gamma \eta_{k})\mid \mathcal{F}_{k-1}) \leq \exp\left( \frac{\gamma^{2}R^{2}}{2}\right) \text{ a.s }
    \end{align}
    Consider the martingale $\xi_{t} = \sum_{k=1}^{t} m_{k-1} \eta_{k}$ and the process $M_{t} = \sum_{k=1}^{t} m_{k-1}m_{k-1}^{\intercal}$. Assume that with probability
    one the smallest of $M_{d}$ is lower bounded by some positive constant $\lambda_{0}$ and that $\| m_{k}\|_{2} \leq c_{m}$ holds a.s. for any $k\geq 0$.
    The following hold true:
    \hspace{1cm} Let $$\kappa = \sqrt{3+2\log(1 + 2c_{m}^{2}/\lambda_{0})}$$
    For any $x\in \mathbb{R}^{d}$, $0< \delta \leq e^{-1}$, $t\geq \max(d,2)$, with probability at least $1 - \delta$,
    \begin{equation}
        |\langle x, \xi_{t}\rangle \leq \kappa R \sqrt{2\log(t)\log(1/\delta)}\| x\|_{M_{t}}
    \end{equation}
    Further, for any $0<\delta< \min\{1, d/e\}$, $t\geq \max\{ d, 2\}$, with probability at least $1 - \delta$,
    \begin{equation}
        \| \xi_{t} \|_{M_{t}^{-1}} \leq \kappa R \sqrt{2d\log(t)\log(d/\delta)}
    \end{equation}

\end{proposition}

In our setting, we can not use the type of results presented above because  estimators like ridge regression or maximum likelihood estimation rely on the assumption that the noise associated to the data is zero-mean. However,
in our setting the learner\footnote{Notice that the following discussion holds also for any non-learning algorithm, e.g., an oracle algorithm.} takes a decision $\gA_t$ based on the noisy evaluations $(\phi_{i,t})_{i}$ and this introduces a statistical dependence that biases the distribution of the noise affecting the evaluations. For simplicity, assume $K = 1$ then the action $a_{t}$ is a function of the noise in the observations $(\phi_{i,t})$ and, in general, we have that  $\mathbb{E}\left(\epsilon_{a_{t},t,j} \right) \neq 0$ for all evaluators $j\leq J$ (where $(\epsilon_{i,t,j})$ is the noise in the evaluation $f_{i,t,j}$) although for any fixed non random action $a\leq K_{t}$, $\mathbb{E}(\epsilon_{a,t,j}) = 0$.

More formally, let consider the GLM setting of Sec.~\ref{sec:general.case} assuming $K = 1$ and a MLE estimator using all the data $\mathcal{H}_{t} := \{ (f_{a_{t},t,j})_{j}, r_{a_{t},t} \}$ to estimate the parameter $\alpha$.
Let also assume that the action $a_{t}$ is chosen as $a_{t} = \arg\max_{i} \langle w_{t}, \phi_{i,t}\rangle$
where for any $t$, $w_{t}$ is a vector adapted to the $\sigma$-algebra $\sigma(\mathcal{H}_{t-1})$. For each evaluator $j$ we compute the MLE, $\wh{\alpha}_{j}$ as the solution to
\begin{align}
\sum_{l=1}^{t-1} (f_{a_{t}, t,j} - g(\wh{\alpha}_j r_{i_t,t}))r_{i_t,t} = 0
\end{align}
with $a_t$ is the arm selected by the algorithm and $r_{a_t,t}$ is the associated reward observed at the end of the round. In order to evaluate the accuracy of this estimator $\wh{\alpha}_{j}$, following the proof of~\citet{filippi2010parametric} or \citep{li2017provably}, we eventually need to control the term
\begin{align}
    \frac{\left|\sum_{l=1}^{t} r_{a_l,l}\epsilon_{a_l,l,j}\right|}{\sqrt{\sum_{l=1}^{t} r_{a_l,l}^{2}}},
\end{align}
where $\epsilon_{a_l,l,j}$ is the noise associated with the evaluation $j$ for arm $a_l$ at round $l$. One may be tempted to apply Prop.~\ref{prop:glm} with $m_{t-1} = r_{a_t,t}$ and $\eta_{t} = \epsilon_{a_t,t,j}$. The issue is that $a_{t} = \arg\max_{i} \langle w_{t-1}, g^{-1}(\phi_{i,t})\rangle $ where $w_{t}$ is a measurable function for of the past (in addition, note that this action selection process is used in most optimistic algorithms for GLM).
As a result,
\begin{align}\label{eq:biased.noise}
    \mathbb{E}\left( \epsilon_{a_{t},t,j}\mid \mathcal{H}_{t-1}, (r_{i,t})_{i\leq K_{t}}\right) \neq 0,
\end{align}
\Eqref{eq:biased.noise} can be further simplified when $g$ is the identity function, $J = 1$, $\alpha_j = 0$, $\sigma_j=\sigma$ and $w_{t} \neq 0$ then \Eqref{eq:biased.noise} reduces to
\begin{align}
    \mathbb{E}\left( \epsilon_{a_t, t, j}\mid \mathcal{H}_{t-1}, (r_{i,t})_{i\leq K_{t}}\right) = \mathbb{E}\left( \epsilon_{i,t}\mathds{1}_{\{i = \arg\max w_{t}\epsilon_{i,t}\}}\mid \mathcal{H}_{t-1}\right) \approx \sigma,
\end{align}
thus showing that $\epsilon_{a_t,t}$ is no longer a zero-mean variable. A way to address this issue would be to take a union bound over all the arms chosen over time $a_1,\ldots,a_t,\ldots,a_T$. Unfortunately, this would lead to take a union bound over $K^T$ elements, which would eventually lead to an additional $\sqrt{T}$ factor in Prop~\ref{prop:glm} and ultimately a linear regret bound. We further confirm this effect in several empirical validations, where algorithms relying on data generated based on the evaluations introduce a bias in the estimation of the parameters $\alpha$.

\section{GENERALIZED LINEAR MODEL}\label{app:glm_proof}

In this section, we present how to defined the oracle and the analysis of the regret in the case of a GLM model.

\subsection{Oracle in the GLM Model (Proof of Lemma~\ref{lem:error_oracle})}\label{app:glm.oracle}

We consider an oracle $\mathfrak{O}$ with access to the link function $g$, the parameters $(\alpha_{j})_{j\leq J}$ of the evaluation function, and the parameters $(\sigma_{j})_{j\leq J}$ of the noise distribution. Given the vector $\phi_{i,t}\in\R^J$ obtained by aggregating all the evaluations  $(f_{i,t,.})_{j}$ for each arm $i\leq K_{t}$, we recall that an oracle defines a set of weights $(w_{j})_{j\leq J}$ and predicts the reward of arm $i$ as
 \begin{equation}
     r_{i,t}^{\mathfrak{O}} = \langle w , g^{-1}(\phi_{i,t})\rangle
 \end{equation}



We aim at minimizing the gap between reward of the true top-K arms $i_1^\star, \ldots, i_K^\star$ and the reward of the \emph{estimated top-K arms} $i_1^\mathfrak{O}, \ldots, i_K^\mathfrak{O}$ according to the estimated rewards $r_{i,t}^{\mathfrak{O}}$
\begin{equation}
\begin{aligned}
(a) = \sum_{l=1}^{K} r_{i_{l}^\star,t} - \sum_{l=1}^{K} r_{i_{l}^\mathfrak{O},t}.
\end{aligned}
\end{equation}
By leveraging the definition of estimated top-K arms, we can rewrite the previous expression as
\begin{align*}
(a) &= \sum_{l=1}^{K} r_{i_{l}^\star,t} - \sum_{l=1}^{K} \widehat{r}_{i_{l}^\star,t}^{\mathfrak{O}} + \sum_{l=1}^{K} \widehat{r}_{i_{l}^\star,t}^{\mathfrak{O}} -
\sum_{l=1}^{K} \widehat{r}_{i_{l}^\mathfrak{O},t}^{\mathfrak{O}} + \sum_{l=1}^{K} \widehat{r}_{i_{l}^\mathfrak{O},t}^{\mathfrak{O}} -
\sum_{l=1}^{K} r_{i_{l}^\mathfrak{O},t} \\
&\leq \sum_{l=1}^{K} r_{i_{l}^\star,t} - \sum_{l=1}^{K} \widehat{r}_{i_{l}^\star,t}^{\mathfrak{O}} + \sum_{l=1}^{K} \widehat{r}_{i_{l}^\mathfrak{O},t}^{\mathfrak{O}} -
\sum_{l=1}^{K} r_{i_{l}^\mathfrak{O},t} \\
&= \sum_{l=1}^{K} \Big( r_{i_{l}^\star,t} - \widehat{r}_{i_{l}^\star,t}^{\mathfrak{O}}\Big) + \sum_{l=1}^{K} \Big( \widehat{r}_{i_{l}^\mathfrak{O},t}^{\mathfrak{O}} -
r_{i_{l}^\mathfrak{O},t} \Big) \\
&\leq 2\max_{i_1,\ldots,i_K}\sum_{l=1}^{K} \underbrace{\Big| \widehat{r}_{i_{l},t}^{\mathfrak{O}} - r_{i_{l},t}\Big|}_{(b)}.
\end{align*}
Using the GLM model of Asm.~\ref{asm:general.setting}, we have that for any $i=1,\ldots,K_t$, the estimated reward can be written as
\begin{equation}
     \widehat{r}_{i,t}^{\mathfrak{O}}= \sum_{j=1}^{J} w_{j}g^{-1}\left(g(\alpha_{j}\cdot r_{i,t}) + \epsilon_{i,t,j}\right).
\end{equation}
We can then measure the deviation between true reward and estimated reward as
\begin{align}
        (b) &= \sum_{j=1}^J w_{j}\Big[g^{-1}\left(g(\alpha_{j}\cdot r_{i,t}) + \epsilon_{i,t,j}\right) - g^{-1}(g(\alpha_{j}r_{i,t}))\Big] + \Big(\sum_{j=1}^{J} w_{j}\alpha_{j} - 1\Big)r_{i,t}\\
        &\leq c_g^{-1} \sum_{j=1}^J |w_{j}| \cdot |\epsilon_{i,t,j}| + \Big(\sum_{j=1}^{J} w_{j}\alpha_{j} - 1\Big)r_{i,t},
\end{align}
where in the first equality we introduce the term $\sum_{j=1}^{K} w_{j}\alpha_{j} =\sum_{j=1}^{K} w_{j} g^{-1}(g(\alpha_{j}r_{i,t}))$ and in the second step we leverage the fact that $g^{-1}$ is $c_{g}^{-1}$-Lipschitz. The last expression above contains two random realizations, i.e., $\epsilon_{i,t,j}$ and $r_{i,t}$, thus preventing from computing a fixed set of weights to minimize the performance gap $(a)$. Furthermore, while we assume the oracle to have prior knowledge about the parameters, we would like to avoid leveraging any knowledge on the true rewards $(r_{i,t})$. In order to remove any dependency on the rewards, we impose a constrain on the weights, such that $\sum_{j=1}^{J} w_{j}\alpha_{j} = 1$, thus removing the last term, and thus the rewards, in the previous expression. We can then focus on minimizing a high-probability upper bound of $(a)$
%
\begin{equation}
    \begin{aligned}
        (a) \leq 2c_g^{-1}\max_{i_1,\ldots,i_K}\sum_{l=1}^{K}\sum_{j=1}^J |w_{j}| \cdot |\epsilon_{i,t,j}| 
    \end{aligned}
\end{equation}
where we leverage the fact that the noise $\epsilon_{i,t,j}$ are sub-Gaussian then $|\epsilon_{i,t,j}| - \mathbb{E}(|\epsilon_{i,t,j}|)$ is also sub-Gaussian with the same parameter and by Jansen inequality its mean is $\mathbb{E}(|\epsilon_{i,t,j}|) \leq \sqrt{\mathbb{E}(\epsilon_{i,t,j}^{2})} = \sigma_{j}$.
Therefore for any time $t$ and any set $U\subset \{1, \cdots, K_{t}\}$ of size $K$ , we have by Chernoff inequality, for any $x>0$
\begin{equation}
    \mathbb{P}\left( \left|\sum_{u\in U}\sum_{j=1}^{J} |w_{j}\epsilon_{u,t,j}| - \mathbb{E}(|w_{j}\epsilon_{u,t,j}|)\right| \geq x\right) \leq 2\exp\left(- \frac{x^{2}}{2K\sum_{j=1}^{J} (w_{j}\sigma_{j})^{2}} \right)
\end{equation}
Therefore we have taking a union over all the set of size $K$,
\begin{align}
    \mathbb{P}\left(\max_{i_{1}, \cdots, i_{K}} \sum_{l=1}^{K}\sum_{j=1}^{J} |w_{j}\epsilon_{i_{l},t,j}| - \mathbb{E}(|w_{j}\epsilon_{i_{l},t,j}|) \geq x \right) &\leq \sum_{U, |U| = K} \mathbb{P}\left(\sum_{u\in U}\sum_{j=1}^{J} |w_{j}\epsilon_{u,t,j}| - \mathbb{E}(|w_{j}\epsilon_{u,t,j}|) \geq x \right)\\
    &\leq \sum_{U, |U| = K} 2\exp\left(- \frac{x^{2}}{2\sum_{j=1}^{J} (w_{j}\sigma_{j})^{2}} \right) \\
    &= 2\binom{K_{t}}{K}\exp\left(- \frac{x^{2}}{2K^{2}\sum_{j=1}^{J} (w_{j}\sigma_{j})^{2}} \right) 
\end{align}
Therefore for any $\delta\in(0,1)$ we have that with probability at least $1 - \delta$:
{\small\begin{align}
    \max_{i_1,\ldots,i_K}\sum_{l=1}^{K}\sum_{j=1}^J |w_{j}| \cdot |\epsilon_{i,t,j}| - \mathbb{E}(|w_{j}| \cdot |\epsilon_{i,t,j}|) \leq K\sqrt{2\sum_{j=1}^{J} (w_{j}\sigma_{j})^{2}\ln\left(\frac{2\binom{K_{t}}{K}}{\delta} \right)} \leq K\sqrt{2K\sum_{j=1}^{J} (w_{j}\sigma_{j})^{2}\ln\left(\frac{2eK_{\text{max}}}{K\delta} \right)}
\end{align}}
using the standard inequality $\binom{K_{t}}{K} \leq \left(\frac{eK_{t}}{K} \right)^{K}$ and $K_{t} \leq K_{\text{max}}$.
Therefore, 
\begin{align}
    (a) &\leq \frac{2K}{c_{g}}\sqrt{K\sum_{j=1}^{J} w_{j}^{2}\sigma_{j}^{2}\ln\left(\frac{K_{t}e}{\delta}\right)} + \frac{K}{c_{g}}\sum_{j=1}^{J} |w_{j}| \sigma_{j}\leq \frac{2K}{c_{g}}\sqrt{K\sum_{j=1}^{J} w_{j}^{2}\sigma_{j}^{2}\ln\left(\frac{K_{\text{max}}e}{\delta}\right)} + \frac{K}{c_{g}}\sqrt{J\sum_{j=1}^{J} (w_{j} \sigma_{j})^{2}},
\end{align}
Finally, this leads to the optimization problem in Eq.~\ref{eq:optimal.weights.glm}. By plugging the optimal solution back into the optimization problem, we also obtain the stated upper bound on the suboptimality gap.

\subsection{Regret of $\epsilon$-greedy Algorithm for GLM Model (Proof of Thm.~\ref{thm:eps.greedy.glm})}

We now move on proving the regret upper bound of Thm.~\ref{thm:eps.greedy.glm}. The first step is to bound the norm of the noisy evaluations at every step. That is the object of the following lemma.

\begin{lemma}\label{lem:features.bound}
Let $\phi_{i,t} = (f_{i,t,1}, \ldots, f_{i,t,j}, \ldots, f_{i,t,J}) \in \R^{J}$ the vector summarizing all the evaluations observed at round $t$. Then for any $\delta\in(0,1)$ with probability $1-\delta$ for any $t\leq T$, for any set $\gA_t$ of $K$ arms chosen in $\{1,\ldots,K_{\max}\}$ possibly adaptive to the evaluations, it holds that
\begin{align*}
    \big\|\sum_{i\in\gA_t} \phi_{i,t}\big\|_{2} \leq \Phi :=  2K\|\sigma\|_{\infty}\left( 2\sqrt{J} + \sqrt{K\ln\left(\frac{2eK_{\text{max}}T}{K\delta}\right)}\right) + 2KJ\|g \|_{\infty}
\end{align*}
where $\|\sigma\|_{\infty} = \max_{j\leq J} \sigma_{j}$ and $\|g\|_{\infty} = \max_{j\leq J, x\in [0, C]} |g(\alpha_{j} x)|$
\end{lemma}

\begin{proof}
    For a time $t\leq T$ and set of size $K$ $\gA_{t}\subset \{1,\cdots, K_{t}\}$ and an evaluator $j\leq J$ the sum of features can be decomposed as,
    \begin{align}
        \left(\sum_{i\in \gA_{t}} \phi_{i,t,j}\right)^{2} = \left(\sum_{i\in \gA_{t}} g(\alpha_{j} r_{i,t}) + \sum_{i\in \gA_{t}}\epsilon_{i,t,j}\right)^{2} &\leq 2\left( \sum_{i\in \gA_{t}} g(\alpha_{j} r_{i,t})\right)^{2} + 2\left(\sum_{i\in \gA_{t}}\epsilon_{i,t,j}\right)^{2}\\
        &\leq 2K^{2}\|g\|_{\infty}^{2} + 2\left(\sum_{i\in \gA_{t}}\epsilon_{i,t,j}\right)^{2}
    \end{align}
Therefore, we just need to bound $ \big\|\sum_{i\in\gA_t} \epsilon_{i,t}\big\|_{2}$ to finish the proof. But thanks to the subGaussian
assumption on the noise $\epsilon$, we have that:
\begin{align}
    \left\|\sum_{i\in \gA_{t}}\epsilon_{i,t} \right\|_{2} \leq 4K\sqrt{J\max_{j\leq J} \sigma_{j}^{2}} + 2K\max_{j\leq J} \sigma_{j} \sqrt{\log\left(\frac{1}{\delta} \right)}
\end{align}
because the vector $\sum_{i\in \gA_{t}} \epsilon_{i,t}$ is $K\max_{j}\sigma_{j}$ sub-Gaussian. The last step is to take a union bound over steps $t\leq T$ and subset of size $K$.
\end{proof}

The first step to analyze the regret of Alg.~\ref{alg:eps.greedy.glm} is to decompose the regret according to the steps where the algorithm selected a totally random set of arms or when it played greedy.
Noting, as in Alg.~\ref{alg:eps.greedy.glm}, $Z_t$ the Bernoulli random variable used by the algorithm to distinguish between exploratory and exploitative steps, the regret is
\begin{equation}
    R_{T} = \underbrace{\sum_{t=1}^{T} Z_{t}\langle w^{+}, \sum_{i\in \gA_{t}^{+}} \phi_{t,i} - \sum_{i\in\gA_{t}} \phi_{i,t}\rangle}_{ :=R_{T,1}} + \underbrace{\sum_{t=1}^{T} (1-Z_{t})\langle w^{+}, \sum_{i\in\gA_{t}^{+}} \phi_{i,t} - \sum_{i\in\gA_{t}} \phi_{i,t}\rangle}_{ := R_{T,2}}
\end{equation}
In the following, we bound each term $R_{T, 1}$ and $R_{T,2}$.

\paragraph{Bounding $R_{T,1}$.} The term we analyze is the regret due to the random steps in the algorithm. Centering the variable $Z_{t}$ we get that $R_{T,1}$ scales as:
\begin{equation}
    \begin{aligned}
        R_{T, 1} &= \sum_{t=1}^{T} (Z_{t} - \epsilon)\langle w^{+}, \sum_{i\in \gA_{t}^{+}} \phi_{t,i} - \sum_{i\in\gA_{t}} \phi_{i,t}\rangle + \epsilon\sum_{t=1}^{T}\langle w^{+}, \sum_{i\in \gA_{t}^{+}} \phi_{t,i} - \sum_{i\in\gA_{t}} \phi_{i,t}\rangle\\
        &\leq  \sum_{t=1}^{T} (Z_{t} - \epsilon)\langle w^{+}, \sum_{i\in \gA_{t}^{+}} \phi_{t,i} - \sum_{i\in\gA_{t}} \phi_{i,t}\rangle + \epsilon T \| w^{+}\| \left( \big\|\sum_{i\in\gA_t^{+}} \phi_{i,t}\big\|_{2} + \big\|\sum_{i\in\gA_t} \phi_{i,t}\big\|_{2}\right)\\
        &\leq  \sum_{t=1}^{T} (Z_{t} - \epsilon)\langle w^{+}, \sum_{i\in \gA_{t}^{+}} \phi_{t,i} - \sum_{i\in\gA_{t}} \phi_{i,t}\rangle + 2\Phi\epsilon T \| w^{+}\|
    \end{aligned}
\end{equation}
thanks to Lemma~\ref{lem:features.bound}. Now, the first term can be bounded thanks to a standard Azuma-Hoeffding inequality as for every time $t\leq T$, $Z_{t}$ is independent
from the evaluations $(\phi_{i,t})_{i}$ and the set of arms $\gA_{t}$ hence considering the filtration $(\mathcal{F}_{t})_{t}$ being the history up until time $t$, we have,
\begin{align}
\forall t\leq T, \text{ } \left|\sum_{t=1}^{T} (Z_{t} - \epsilon)\langle w^{+}, \sum_{i\in \gA_{t}^{+}} \phi_{t,i} - \sum_{i\in\gA_{t}} \phi_{i,t}\rangle\right| \leq 4\Phi\|w^{+}\|\sqrt{2T\ln\left(\frac{4T}{\delta}\right)}\text{ w.p at least $1 - \delta/2$}
\end{align}
Putting everything toghether we have that for any $\delta\in (0,1)$,
\begin{align}
    R_{T,1} \leq 4\Phi\|w^{+}\|\left(\sqrt{2T\ln\left(\frac{4T}{\delta}\right)} +  \epsilon T\right)
\end{align}
with probability at least $1 - \delta/2$.

\paragraph{Bounding $R_{T,2}$.} Next, we bound the regret coming due to the estimation error in the exploitative steps, that is to say when $Z_{t} = 0$. We begin by using the greedy behavior of the algorithm
to enough to study the error $\| w_{t} - w^{+}\|$ to bound $R_{T,2}$,

\begin{equation}
    R_{T,2} \leq \sum_{t=1}^{T} (1-Z_{t})\left[ \langle w^{+} - w_{t}, \sum_{i\in \gA_{t}^{+}} \phi_{i,t}\rangle + \langle w_{t}, \sum_{i\in \gA_{t}^{+}} \phi_{i,t} -  \sum_{i\in \gA_{t}} \phi_{i,t}\rangle + \langle  w_{t} - w^{+}, \sum_{i\in \gA_{t}} \phi_{i,t}\rangle\right]
\end{equation}
But because $\gA_{t}$ is the maximizer of the estimated reward with respect to $w_{t}$ we have that
\begin{align}
    (1 - Z_{t})\left\langle w_{t}, \sum_{i\in \gA_{t}^{+}} \phi_{i,t} -  \sum_{i\in \gA_{t}} \phi_{i,t}\right\rangle \leq 0.
\end{align}
Furthermore, using Lemma~\ref{lem:features.bound} we have that with probability at least $1 - \delta/4$
\begin{align}
    \max\left\{\left\langle w^{+} - w_{t}, \sum_{i\in \gA_{t}^{+}} \phi_{i,t}\right\rangle, \left\langle  w_{t} - w^{+}, \sum_{i\in \gA_{t}} \phi_{i,t}\right\rangle \right\} \leq \Phi \| w^{+} - w_{t} \|_{2}.
\end{align}

As a result, the remaining step to obtain a bound on the regret is to build a concentration inequality for the weights $w_{t}$ w.r.t.\ $w^{+}$ and properly tune the exploration factor $\epsilon$.

We start by studying the concentration of the MLE estimator $\wh\alpha_{t}\in\R^J$ to the true parameter $\alpha\in\R^J$.
Thanks to the random choice of $\gA_t$ in the explorative steps, all the data in $\mathcal{H}_{t}$ are i.i.d.\ and we can directly apply the standard concentration inequality for MLE in GLM~\citep{li2017provably}.
\begin{lemma}\label{lem:error.alpha}
For any $\delta\in (0,1)$ with probability at least $1 - \delta/8$ for all $j\leq J$ and $t\leq T$:
\begin{equation}
    \begin{aligned}
        \left| \wh\alpha_{t,j} - \alpha_{j}\right| \leq \beta_{t,j} := \frac{
        \sqrt{\lambda}\|\alpha\| + \sigma_{j}\sqrt{\frac{1}{2}\log\left(1 + \frac{2T}{\lambda}\right) +\log\left(\frac{8JT}{\delta}\right)}}{c_{g}\sqrt{\sum_{l, Z_{l} = 1}\sum_{i\in \gA_{t}} r_{l,i}}}.
    \end{aligned}
\end{equation}
In other words,
\begin{align}
    \left\| \wh{\alpha}_{t} - \alpha_{r}\right\|_{\infty} \leq \beta_{t} :=\max_{j} \beta_{t,j} =
        \frac{\sqrt{\lambda}\|\alpha\| + \|\sigma\|_{\infty}\sqrt{\frac{1}{2}\log\left(1 + \frac{2T}{\lambda}\right) +\log\left(\frac{8JT}{\delta}\right)}}{c_{g}\sqrt{\sum_{l, Z_{l} = 1}\sum_{i\in \gA_{t}} r_{l,i}}}
\end{align}
\end{lemma}


We leverage Lem.~\ref{lem:error.alpha} to concentrate $\| w^{\star} - w_{t} \|$. But first, recall that $w_{t}$ is given by:
\begin{equation*}
    \begin{aligned}
        w_{t} = \frac{\sigma^{-1}\cdot\wh{\alpha}_{t}}{\| \sigma^{-1}\cdot\wh{\alpha}_{t}\|^{2}_{2}} \text{ and } w^{+} = \frac{\sigma^{-1}\cdot\alpha}{\| \sigma^{-1}\cdot\alpha\|^{2}_{2}}
    \end{aligned}
\end{equation*}
where $\Sigma = \text{diag}(\sigma_{1}, \cdots, \sigma_{J})$. Therefore, the error between $w^{+}$ and $w_{t}$ can be written as:
\begin{align}\label{eq:error_weights}
    \| w_{t} - w^{+}\|_{2} &= \left\|\frac{\sigma^{-2}\cdot\left( \wh{\alpha}_{t} - \alpha\right)}{\|\sigma^{-1}\cdot\wh{\alpha}_{t}\|^{2}} + \sigma^{-2}\cdot\alpha\left(\frac{1}{\| \sigma^{-1}\cdot\wh{\alpha}_{t}\|^{2}} -  \frac{1}{\| \sigma^{-1}\cdot\alpha\|^{2}}\right) \right\| \\
    &\leq \left\|\sigma^{-2}\cdot\alpha\left(\frac{1}{\| \sigma^{-1}\cdot\wh{\alpha}_{t}\|^{2}} -  \frac{1}{\| \sigma^{-1}\cdot\alpha\|^{2}}\right)\right\|_{2} + \left\|\frac{\sigma^{-2}\cdot\left( \wh{\alpha}_{t} - \alpha\right)}{\|\sigma^{-1}\cdot\wh{\alpha}_{t}\|^{2}}  \right\|_{2}
\end{align}
\begin{proposition}\label{prop:bound.first.term}
For any time $t\leq T$, assuming that $\|\wh{\alpha}_{t} - \alpha\|_{\infty} \leq \beta_{t}$ and $\|\alpha\|_{\infty} \geq 8\beta_{t}$ then:
\begin{align}
        \frac{\|\sigma^{-2}\cdot\left(\wh{\alpha}_{t} - \alpha\right)\|_{2}}{\langle \wh{\alpha}_{t}, \sigma^{-2}\cdot\wh{\alpha}_{t}\rangle} \leq \frac{4\beta_{t}\sqrt{\sum_{j=1}^{J} \frac{1}{\sigma_{j}^{4}}}}{\sum_{j=1}^{J} \frac{\alpha_{j}^{2}\mathds{1}_{\{|\alpha_{j}| \geq 8\beta_{t}\}}}{\sigma^{2}_{j}}} = 4\beta_{t}\left(\frac{\| \sigma^{-1}\|_{4}}{\| \sigma^{-1}\cdot\alpha\mathds{1}_{\{|\alpha|\geq 8\beta_{t}\}}\|_{2}}\right)^{2}
\end{align}
\end{proposition}

\begin{proof}
Using the definition of $\sigma^{-2}\cdot$ we have,
\begin{align}\label{eq:ratio}
    \frac{\|\sigma^{-2}\cdot\left(\wh{\alpha}_{t} - \alpha\right)\|_{2}}{\langle \wh{\alpha}_{t}, \sigma^{-2}\cdot\wh{\alpha}_{t}\rangle} = \frac{\sqrt{\sum_{j=1}^{J}\frac{(\wh{\alpha}_{t,j} - \alpha_{j})^{2}}{\sigma^{4}_{j}}}}{\sum_{r=1}^{J} \frac{\wh{\alpha}_{t,j}^{2}}{\sigma_{j}^{2}}}
\end{align}
However, we have that for any $j\leq J$:
\begin{align*}
    \wh{\alpha}_{t,j}^{2} = \alpha_{j}^{2} + \left(\wh{\alpha}_{t,j} - \alpha_{j} \right)^{2} + 2\left( \wh{\alpha}_{t,j} - \alpha_{j}\right)\alpha_{j} &\geq \alpha_{j}^{2} + 2\left( \wh{\alpha}_{t,j} - \alpha_{j}\right)\alpha_{j} \\
    &\geq \alpha_{j}^{2} - 2\beta_{t}|\alpha_{j}| \\
    &\geq \frac{\alpha^{2}_{j}}{2} \text{ when } |\alpha_{j}| \geq 8\beta_{t}
\end{align*}
using that $\|\wh{\alpha}_{t} - \alpha\|_{\infty} \leq \beta_{t}$. In addition, using again that $\|\wh{\alpha}_{t} - \alpha\|_{\infty} \leq \beta_{t}$,
\begin{align}
    \sum_{j=1}^{J} \frac{\left( \wh{\alpha}_{t,j} - \alpha_{j}\right)^{2}}{\sigma_{j}^{4}} \leq \sum_{j=1}^{J} \frac{\beta_{t}^{2}}{\sigma_{j}^{4}}
\end{align}
\end{proof}

Using the same reasonning as in Prop.~\ref{prop:bound.first.term}, we can show the following which bound the second term in \Eqref{eq:error_weights}.
\begin{proposition}\label{prop:bound.second.term}
    For any time $t\leq T$, assuming that $\|\wh{\alpha}_{t} - \alpha\|_{\infty} \leq \beta_{t}$ and $\|\alpha\|_{\infty} \geq 8\beta_{t}$ then:
    \begin{equation}
        \begin{aligned}
            \left\|\left( \frac{\|\sigma^{-1}\cdot\alpha\|^{2} - \|\sigma^{-1}\cdot\wh{\alpha}_{t}\|^{2}}{\left\|\sigma^{-1}\cdot\wh{\alpha}_{t}\right\|^{2}}\right)\frac{\sigma^{-2}\cdot\alpha}{\|\sigma^{-1}\cdot\alpha\|^{2}}\right\| \leq
            \frac{4\beta_{t}\left( \| \sigma^{-2}\cdot\alpha\|_{2} + \beta_{t}\|\sigma^{-2}\|\right)\|\sigma^{-2}\cdot\alpha\|}{\|\sigma^{-1}\cdot\alpha\mathds{1}_{\{| \alpha \| \geq 8\beta_{t} \}} \|^{2}\| \sigma^{-1}\cdot\alpha\|^{2}}
        \end{aligned}
    \end{equation}
\end{proposition}

Therefore, using Prop.~\ref{prop:bound.first.term} and Prop.~\ref{prop:bound.second.term} on the event that the confidence intervals in Lem.~\ref{lem:error.alpha} holds then with probability at least
$1 - \delta/4$:
\begin{align}
    \| w_{t} - w^{+}\|_{2} \leq \frac{4\beta_{t}\left( \| \sigma^{-2}\cdot\alpha\|_{2} + \beta_{t}\|\sigma^{-2}\|\right)\|\sigma^{-2}\cdot\alpha\|}{\|\sigma^{-1}\cdot\alpha\mathds{1}_{\{| \alpha | \geq 8\beta_{t} \}} \|^{2}\| \sigma^{-1}\cdot\alpha\|^{2}} + 4\beta_{t}\left(\frac{\| \sigma^{-1}\|_{4}}{\| \sigma^{-1}\cdot\alpha\mathds{1}_{\{|\alpha|\geq 8\beta_{t}\}}\|_{2}}\right)^{2}
\end{align}


In order to complete the proof of the regret upper bound, we need to control how fast the confidence intervals width $\beta_{t}$ decreases with $t$.

\begin{lemma}\label{lem:speed.beta.t}
For any $\delta\in (0,1)$ and $\lambda> 0$ we have with probability at least $1 - \delta/4$,
    \begin{equation}
        \forall t\in \left\{\frac{2C^{4}\ln\left(\frac{8T}{\delta}\right)}{\epsilon^{2}K^{2}\eta_{\nu}^{4}},\hdots, T\right\}, \text{ } \beta_{t} \leq \frac{\|\sigma\|_{\infty}\sqrt{\ln\left(1 + \frac{2T}{\lambda}\right) +\ln\left(\frac{8TJ}{\delta}\right)} + \sqrt{\lambda}\|\alpha\|}{c_{g}\sqrt{\epsilon t K\eta_{\nu}^{2} -  C^{2}\sqrt{2t\ln\left(\frac{8T}{\delta}\right)}}}
    \end{equation}
    where for all $j\leq J$, $\eta_{\nu}^{2} = \mathbb{E}_{r\sim\nu}(r^{2})$
\end{lemma}
\begin{proof}
First, thanks to a standard Hoeffding inequality we have that with probability at least $1 - \delta/4$,
\begin{equation*}
    \forall t\leq T, \text{ } \left|\sum_{l\leq t, Z_{l} = 1} \sum_{i\in A_{t}}r_{l,i}^{2} - \mathbb{E}\left(\sum_{l\leq t, Z_{l} = 1} \sum_{i\in A_{t}}r_{l,i}^{2}\right) \right| \leq C^{2}\sqrt{2t\ln\left(\frac{8T}{\delta}\right)}
\end{equation*}
And $\mathbb{E}\left(\sum_{l\leq t, Z_{l} = 1} \sum_{i\in A_{t}}r_{l,i}^{2}\right) = \epsilon t K\eta_{\nu}^{2}$ where $\eta_{\nu}^{2} := \mathbb{E}_{r\sim \nu}(r^{2})$.
Therefore, using the definition of $\beta_{t}$
\begin{equation}
    \beta_{t} \leq \frac{\|\sigma\|_{\infty}\sqrt{\ln\left(1 + \frac{2T}{\lambda}\right) +\ln\left(\frac{8TJ}{\delta}\right)} + \sqrt{\lambda}\|\alpha\|}{c_{g}\sqrt{\epsilon t K\eta_{\nu}^{2} -  C^{2}\sqrt{2t\ln\left(\frac{8T}{\delta}\right)}}}
\end{equation}
\end{proof}
As a consequence, for $t\geq \tau_{0} := \frac{8\ln\left(\frac{8T}{\delta}\right)C^{4}}{\big(\epsilon K\eta_{\nu}^{2}\big)^{2}}$, $\beta_{t} \leq \mathcal{O}\left( \frac{1}{\sqrt{\epsilon t}}\right)$.
Therefore, using Lem.~\ref{lem:speed.beta.t} with probability at least $1 - \delta/4$:
\begin{equation}
    \begin{aligned}
        \sum_{t=\tau_{0} + 1}^{T} \beta_{t} &\leq \sum_{t=\tau_{0} + 1}^{T} \frac{\sqrt{2}\|\sigma\|_{\infty}\sqrt{\ln\left(1 + \frac{2T}{\lambda}\right) +\ln\left(\frac{8TJ}{\delta}\right)} + \sqrt{\lambda}\|\alpha\|}{c_{g}\sqrt{\epsilon t K\eta_{\nu}^{2}}}\\
        &\leq  \frac{3\sqrt{2T}\|\sigma\|_{\infty}}{c_{g}\sqrt{\epsilon K\eta_{\nu}^{2}}}\left(\sqrt{\ln\left(1 + \frac{2T}{\lambda}\right) +\ln\left(\frac{8TJ}{\delta}\right)} + \sqrt{\lambda}\|\alpha\|\right)
    \end{aligned}
\end{equation}
We can finally bound $R_{T,1}$ with probability at least $1 - \delta/2$ by:
\begin{equation}
    \begin{aligned}
R_{T,1} \leq \tau + \frac{24\Phi\sqrt{2T}\|\sigma\|_{\infty}}{c_{g}\sqrt{\epsilon K\eta_{\nu}^{2}}}\left(\sqrt{\ln\left(1 + \frac{2T}{\lambda}\right) +\ln\left(\frac{8TJ}{\delta}\right)} + \sqrt{\lambda}\|\alpha\|\right)\Bigg( \frac{\left( \| \sigma^{-2}\cdot\alpha\|_{2} + \|\sigma^{-2}\|\right)\|\sigma^{-2}\cdot\alpha\|}{\|\sigma^{-1}\cdot\alpha\mathds{1}_{\{| \alpha | \geq 8\beta_{t} \}} \|^{2}\| \sigma^{-1}\cdot\alpha\|^{2}}&
\\+ \left(\frac{\| \sigma^{-1}\|_{4}}{\| \sigma^{-1}\cdot\alpha\mathds{1}_{\{|\alpha|\geq 8\beta_{t}\}}\|_{2}}\right)^{2}\Bigg)&
    \end{aligned}
\end{equation}
where $\tau := \max\left\{\frac{128\|\sigma\|_{\infty}\left( \ln\left(1 + \frac{2T}{\lambda}\right) + \ln\left(\frac{8TJ}{\delta}\right)\right) + \lambda\|\alpha\|^{2}}{\min_{j}|\alpha_{r}|^{2}\epsilon K\eta_{\nu}^{2}}, \frac{8\ln\left(\frac{8T}{\delta}\right)C^{4}}{\big(\epsilon K\eta_{\nu}^{2}\big)^{2}}\right\}$

Finally, the regret is bounded with probability at least $1 - \delta$ by:
{\small\begin{equation}\label{eq:regret_bound}
    \begin{aligned}
        R_{T} \leq \frac{4\Phi}{\| \sigma^{-1}\cdot\alpha\|}\left(\sqrt{2T\ln\left(\frac{4T}{\delta}\right)} +  \epsilon T\right) + \max\left\{\frac{128\|\sigma\|_{\infty}\left( \ln\left(1 + \frac{2T}{\lambda}\right) + \ln\left(\frac{8TJ}{\delta}\right)\right)+ \lambda\|\alpha\|^{2}}{\min_{j, \alpha_{j}\neq 0}|\alpha_{j}|^{2}\epsilon K\eta_{\nu}^{2}}, \frac{8\ln\left(\frac{8T}{\delta}\right)C^{4}}{\big(\epsilon K\eta_{\nu}^{2}\big)^{2}}\right\}&\\
        + \frac{24\Phi\sqrt{2T}\|\sigma\|_{\infty}}{c_{g}\sqrt{\epsilon K\eta_{\nu}^{2}}}\left(\sqrt{\ln\left(1 + \frac{2T}{\lambda}\right) +\ln\left(\frac{8TJ}{\delta}\right)} + \sqrt{\lambda}\|\alpha\|\right)\Bigg( \frac{\left( \| \sigma^{-2}\cdot\alpha\|_{2} + \|\sigma^{-2}\|\right)\|\sigma^{-2}\cdot\alpha\|}{\|\sigma^{-1}\cdot\alpha \|^{2}\| \sigma^{-1}\cdot\alpha\|^{2}} + \left(\frac{\| \sigma^{-1}\|_{4}}{\| \sigma^{-1}\cdot\alpha\|_{2}}\right)^{2}\Bigg)&
    \end{aligned}
\end{equation}}

\textbf{Remark.} Compared to the bound reported in Thm.~\ref{thm:eps.greedy.glm}, we corrected an error in the analysis which removed the dependency on the norm $\| g\|_{\infty}$ and replaceed it
with the upper bound on the distribution of the reward $C$. In addition, here we report a bound depends on $\min_{j\leq J, \alpha_{j} \neq 0} |\alpha_{j}|^{2}$ however because of the condition $\mathds{1}_{\{|\alpha_{j}| \geq 8\beta_{t}\}}$ it is possible to present a regret upper bound
that scales inversly with $\max_{j} | \alpha_{j}|$ but at the cost of a much bigger problem-dependent quantity $S$. Fortunately, the difference between the bound report in Thm~\ref{thm:eps.greedy.glm} and the one in \Eqref{eq:regret_bound}  has no impact on the discussion around the depency on $J$. Simplifying the bound in \Eqref{eq:regret_bound}, we get that the regret is bounded with high probability by,
\begin{align}
    R_{T} \leq \wt{\mathcal{O}}\Bigg(T^{2/3}\Bigg(\frac{(1+\sqrt{J^{-1}}\|\alpha\|)\Phi S\|\sigma\|_{\infty}}{c_{g}\sqrt{K}\eta_{\nu}} + \max\left\{\frac{\|\sigma\|_{\infty}}{\min_{j, \alpha_{j}\neq 0}|\alpha_{j}|^{2}K\eta_{\nu}^{2}}, \frac{C^{4}}{\big(K\eta_{\nu}^{2}\big)^{2}}\right\}\Bigg) 
\end{align}
when choosing $\epsilon = \frac{1}{T^{1/3}}$, $\lambda = J^{-1}$ and setting $S := \frac{\left( \| \sigma^{-2}\cdot\alpha\|_{2} + \|\sigma^{-2}\|\right)\|\sigma^{-2}\cdot\alpha\|}{\|\sigma^{-1}\cdot\alpha \|^{2}\| \sigma^{-1}\cdot\alpha\|^{2}} + \left(\frac{\| \sigma^{-1}\|_{4}}{\| \sigma^{-1}\cdot\alpha\|_{2}}\right)^{2}$.

\section{LINEAR MODEL}\label{app:affine_model}

In this section, we show how the definition of the oracle changes with the linear setting and the proof of Thm.~\ref{thm:greedy.linear}, that is to the regret of Alg.~\ref{alg:greedy.affine}

\subsection{Oracle in the Linear Model (Proof of Lem.~\ref{lem:error_oracle_affine})}

Just as in \Secref{app:glm.oracle}, we consider an oracle that has access to the parameters $(\alpha_{j})_{j\leq J}$  but also knows the
noise $(\sigma_{j})_{j\leq J}$ for each evaluation functions. Given a vector of evalutations $\phi_{i,t}$ for some $t\leq T$ and $i\leq K_{t}$, an orcale is defined through
a set of weights $(w_{j})_{j\leq J}$ and predicts the reward, $r_{i,t}^{\mathfrak{O}} = \langle w, \phi_{i,t}\rangle$. Following the reasonning
in \Secref{app:glm.oracle}, the difference between the true top-$K$ arms, $i_{1}^{+}, \cdots, i_{K}^{+}$, and the top $K$ according to the oracle, $i_{1}^{\mathfrak{O}}, \cdots, i_{K}^{\mathfrak{O}}$
is bounbed by
\begin{equation}
    \begin{aligned}
        \sum_{l=1}^{K} r_{i_{l}^{+}, t} - r_{i_{l}^{\mathfrak{O}}, t} \leq 2 \max_{i_{1}, \cdots, i_{K}} \sum_{l=1}^{K} \left| \wh{r}_{i_{l},t}^{\mathfrak{O}} -  r_{i_{l},t}\right|
    \end{aligned}
\end{equation}
Again following the same reasonning as in App.~\ref{app:glm.oracle}, the difference between the estimated reward and the actual reward for any arm $i\leq K_{t}$ is bounded by
\begin{equation}
    \begin{aligned}
        \left| r_{i,t} - \wh{r}_{i,t}^{\mathfrak{O}} \right| &\leq \left|r_{i,t}\left(\sum_{j=1}^{J} w_{j}\alpha_{j} - 1 \right)\right| + \left| \sum_{j=1}^{J} w_{j}\epsilon_{i,t,j}\right|\leq C\left|\left(\sum_{j=1}^{J} w_{j}\alpha_{j} - 1 \right)\right| + \left| \sum_{j=1}^{J} w_{j}\epsilon_{i,t,j}\right|\\
    \end{aligned}
\end{equation}
The main difference compared to the proof of Lem.~\ref{lem:error_oracle} is that thanks to the linear strucutre contorlling the noise can be done direclty without usingany Lipschitz property.
Therefore, in order to remove the dependency on the reward, we focus on weights $w$ such that $\langle w, \alpha\rangle = 1$.
Hence using Chernoff inequality the error is bounded for any $\delta\in (0,1)$ by
\begin{align}
    \left| r_{i,t} - r_{i,t}^{\mathfrak{O}} \right| \leq 2K\sqrt{K\sum_{j=1}^{J} w_{j}^{2}\sigma_{j}^{2}\ln\left(\frac{K_{\text{max}}e}{\delta}\right)}
\end{align}
with probability at least $1 - \delta$. This conlcudes the proof of Lem.~\ref{lem:error_oracle}

\subsection{Regret of \textsc{ESAG} (Proof of Thm.~\ref{thm:greedy.linear})}

The first step to analyze the regret of Alg.~\ref{alg:greedy.affine} is to compute an upper bound on the deviation of the estimator $\wh{\alpha}_{t}$.

\begin{lemma}\label{lem:error_alpha_linear}
    For any $\delta\in(0,1)$ and $t\leq T$, the error between $\wh{\alpha}_{t}$ and $\mathbb{E}_{r\sim\nu}(r)\alpha$ is bounded with
    probability at least $1 - \delta/4$ by,
    \begin{equation}\label{eq:ap2.ci.alpha}
        \begin{aligned}
            \left\| \wh{\alpha}_{t} - \mathbb{E}_{r\sim\nu}(r)\alpha \right\| \leq \beta_{t}^{L} := \frac{C\sqrt{2\ln\left(\frac{16JT}{\delta}\right)}}{\sqrt{\sum_{l=1}^{t-1} K_{l}}} + \frac{4\|\sigma\|_{\infty}\sqrt{J}}{\sqrt{\sum_{l=1}^{t-1} K_{l}}} + \frac{2\|\sigma\|_{\infty}\sqrt{\ln\left(\frac{8TJ}{\delta}\right)}}{\sqrt{\sum_{l=1}^{t-1} K_{l}}}
        \end{aligned}
    \end{equation}
\end{lemma}

\begin{proof}
For any $t\leq T$, $\wh{\alpha}_{t} = \frac{\sum_{l=1}^{t-1} \sum_{i=1}^{K_{l}} \phi_{i,l}}{\sum_{l=1}^{t-1} K_{l}}$ but using the strucuture of the evaluations $\phi_{i,t}$ for every $j\leq J$,
\begin{equation}\label{eq:error_coordinate_alpha}
    \wh{\alpha}_{t,j} = \frac{1}{\sum_{l=1}^{t-1} K_{l}}\left( \sum_{l=1}^{t-1}\sum_{i=1}^{K_{l}} \alpha_{j}r_{i,l} + \sum_{l=1}^{t-1}\sum_{i=1}^{K_{l}} \epsilon_{i,t,j}\right)
\end{equation}
But the first term converges toward $\alpha_{j}\mathbb{E}_{r\sim\nu}(r)$ because at every time step the reward are independtly sampled from the distribution $\nu$ (see Assumption~\ref{asm:rewards}). More precisely, with probability at least $1 - \delta/8$,
for any $t\leq T$ and $j\leq J$
\begin{equation}
    \begin{aligned}
        \left| \frac{1}{\sum_{l=1}^{t-1} K_{l}}\sum_{l=1}^{t-1}\sum_{i=1}^{K_{l}} r_{i,l} - \mathbb{E}_{r\sim\nu}(r)\right|\leq \frac{C\sqrt{2\ln\left(\frac{16JT}{\delta}\right)}}{\sqrt{\sum_{l=1}^{t-1} K_{l}}}
    \end{aligned}
\end{equation}
In addition, the second term in Eq.~(\ref{eq:error_coordinate_alpha}) can be bounded with probability at least $1 - \delta/ 8$ by
\begin{equation}
    \begin{aligned}
        \left\| \frac{1}{\sum_{l=1}^{t-1} K_{l}}\sum_{l=1}^{t-1}\sum_{i=1}^{K_{l}} \epsilon_{i,l} \right\|  &\leq \frac{4\sqrt{J\max_{j} \sigma_{j}^{2}\sum_{l=1}^{t-1} K_{l} } + 2\sqrt{\max_{j} \sigma_{j}^{2} \sum_{l=1}^{t-1} K_{l} \ln\left(\frac{8TJ}{\delta}\right)}}{\sum_{l=1}^{t-1} K_{l}}\\
        &= \frac{4\|\sigma\|_{\infty}\sqrt{J}}{\sqrt{\sum_{l=1}^{t-1} K_{l}}} + \frac{2\|\sigma\|_{\infty}\sqrt{\ln\left(\frac{8TJ}{\delta}\right)}}{\sqrt{\sum_{l=1}^{t-1} K_{l}}}
    \end{aligned}
\end{equation}
Finally, to finish the proof, we simply need to note that
\begin{align}
    \left\| \wh{\alpha}_{t} - \mathbb{E}_{r\sim\nu}(r)\alpha \right\| &\leq \left\| \frac{\alpha}{\sum_{l=1}^{t-1} K_{l}}\sum_{l=1}^{t-1}\sum_{i=1}^{K_{l}} r_{i,l} - \alpha\mathbb{E}_{r\sim\nu}(r)\right\|+  \left\| \frac{1}{\sum_{l=1}^{t-1} K_{l}}\sum_{l=1}^{t-1}\sum_{i=1}^{K_{l}} \epsilon_{i,l} \right\|\\
    &\leq \left\|\alpha\right\| \left|\frac{1}{\sum_{l=1}^{t-1} K_{l}}\sum_{l=1}^{t-1}\sum_{i=1}^{K_{l}} r_{i,l} - \mathbb{E}_{r\sim\nu}(r)\right|+  \frac{4\|\sigma\|_{\infty}\sqrt{J}}{\sqrt{\sum_{l=1}^{t-1} K_{l}}} + \frac{2\|\sigma\|_{\infty}\sqrt{\ln\left(\frac{8TJ}{\delta}\right)}}{\sqrt{\sum_{l=1}^{t-1} K_{l}}}
\end{align}
\end{proof}

Once \textsc{ESAG} has computed the weigths $\wh{\alpha}_{t}$ it then computes a set of weights thanks to this estimator and based on the shape of the
optimal weigths of the oracle. That is to say the weights $w$ used by the algorithm are:
    \begin{equation}
        \begin{aligned}
            w_{t} = \frac{\sigma^{-2}\cdot\wh{\alpha}_{t}}{\| \sigma^{-1}\cdot\wh{\alpha}_{t}\|_{2}^{2}}
        \end{aligned}
    \end{equation}
    The optimal weights for the parameter $\mathbb{E}_{r\sim\nu}(r)\alpha$ is $\frac{w^{+}}{\mathbb{E}_{r\sim\nu}(r)}$. Therefore, in the following we bound the error between the weights $w_{t}$ and the rescaled optimal weights $\frac{w^{+}}{\mathbb{E}_{r\sim\nu}(r)}$.
    Following the same reasonning as in Prop.~\ref{prop:bound.first.term} and Prop.~\ref{prop:bound.second.term}, we can show the following proposition
    \begin{proposition}\label{prop:confidence_intervals_weights_linear}
        For any $t\leq T$ let's assume $\mathbb{E}_{r\sim\nu}(r)\|\alpha\|_{\infty} \geq 8\beta_{t}^{L}$
        with $\beta_{t}^{L}$ defined in Eq.~(\ref{eq:error_coordinate_alpha}). Then for any $\delta\in (0,1)$ with probability at least $1 - \delta/4$
        {\small\begin{align}\label{eq:confidence_interval_weight_linear}
            \left\| w_{t} - \frac{w^{+}}{\mathbb{E}_{r\sim\nu}(r)}\right\|_{2} \leq \frac{4\beta_{t}^{L}\left( \mathbb{E}_{r\sim\nu}(r)\| \sigma^{-2}\cdot\alpha\|_{2} + \beta_{t}^{L}\|\sigma^{-2}\|\right)\|\sigma^{-2}\cdot\alpha\|}{\mathbb{E}_{r\sim\nu}(r)^{3}\|\sigma^{-1}\cdot\alpha\mathds{1}_{\{\mathbb{E}_{r\sim\nu}(r)|\alpha| \geq 8\beta_{t}^{L} \}} \|^{2}\| \sigma^{-1}\cdot\alpha\|^{2}}
            + \frac{4\beta_{t}^{L}}{\mathbb{E}_{r\sim\nu}(r)^{2}}\left(\frac{\| \sigma^{-1}\|_{4}}{\| \sigma^{-1}\cdot\alpha\mathds{1}_{\{\mathbb{E}_{r\sim\nu}(r)|\alpha|\geq 8\beta_{t}^{L}\}}\|_{2}}\right)^{2}
        \end{align}}
    \end{proposition}

    Thanks to Lemma~\ref{lem:error_alpha_linear} and Prop.~\ref{prop:confidence_intervals_weights_linear}, we can now analyze the regret of \textsc{ESAG}. Recall the regret is defined as follows,
\begin{align}
    R_{T}= \sum_{t=1}^{T} \sum_{i\in \gA_{t}^{+}} \langle w^{+}, \phi_{i,t}\rangle - \sum_{i\in \gA_{t}} \langle w^{+}, \phi_{i,t} \rangle
\end{align}
where $\gA_{t}^{+} := \arg\max_{i}^{K} \langle w^{+}, \phi_{i,t}\rangle$ and $\gA_{t}$ is the set of jobs selected by the learner at time $t$.

However, the set $\gA_{t}^{+}$ is only defined as the arg max of some values, it is
invariant by multiplicating the weights $w^{+}$ by some constant independent of the arms. In particular, we have that for any time $t\leq T$
\begin{align}
    \gA_{t}^{+} = \arg\max_{i}^{K} \langle w^{+}, \phi_{i,t}\rangle = \arg\max_{i}^{K}\left\langle \frac{w^{+}}{\mathbb{E}_{r\sim\nu}(r)}, \phi_{i,t}\right\rangle
\end{align}
Hence the regret can rewritten as,
{\small\begin{align*}
    R_{T} &= \sum_{t=1}^{T} \mathbb{E}_{r\sim\nu}(r) \left(\sum_{i\in \gA_{t}^{+}} \left\langle \frac{w^{+}}{\mathbb{E}_{r\sim\nu}(r)}, \phi_{i,t}\right\rangle - \sum_{i\in \gA_{t}} \left\langle \frac{w^{+}}{\mathbb{E}_{r\sim\nu}(r)}, \phi_{i,t} \right\rangle\right)\\
    &= \sum_{t=1}^{T} \mathbb{E}_{r\sim\nu}(r) \left( \sum_{i\in \gA_{t}^{+}} \left\langle \frac{w^{+}}{\mathbb{E}_{r\sim\nu}(r)} - w_{t}, \phi_{i,t}\right\rangle
    + \sum_{i\in \gA_{t}^{+}} \left\langle w_{t}, \phi_{i,t}\right\rangle  - \sum_{i\in \gA_{t}}\left\langle w_{t}, \phi_{i,t} \right\rangle + \sum_{i\in \gA_{t}} \left\langle w_{t} - \frac{w^{+}}{\mathbb{E}_{r\sim\nu}(r)}, \phi_{i,t}\right\rangle
    \right)
\end{align*}}
But for any $t\leq T$, because of the greedy arm selection process of \textsc{ESAG}:
\begin{align}
    \sum_{i\in \gA_{t}^{+}} \left\langle w_{t}, \phi_{i,t}\right\rangle  - \sum_{i\in \gA_{t}}\left\langle w_{t}, \phi_{i,t} \right\rangle\leq 0
\end{align}
But using Lem.~\ref{lem:features.bound} with probability at least $1 - \delta / 4$
\begin{align}
    \max\left\{\left\langle \frac{w^{+}}{\mathbb{E}_{r\sim\nu}(r)} - w_{t}, \sum_{i\in \gA_{t}^{+}} \phi_{i,t}\right\rangle, \left\langle  w_{t} -\frac{w^{+}}{\mathbb{E}_{r\sim\nu}(r)}, \sum_{i\in \gA_{t}} \phi_{i,t}\right\rangle \right\} \leq \Phi \|\frac{w^{+}}{\mathbb{E}_{r\sim\nu}(r)} - w_{t} \|_{2}.
\end{align}

Thanks to Lem.\ref{lem:error_alpha_linear} and that for all $t\leq T$, $K_{t} \geq K$ therefore for
$t\geq  t_{0} := 1 + 4K^{-1}\left( C^{2}\ln\left(\frac{16JT}{\delta}\right) + 2\|\sigma\|_{\infty}\left( 4J + \ln\left( \frac{8JT}{\delta}\right)\right)\right)\max\left\{\frac{64}{\mathbb{E}_{r\sim\nu}(r)^{2}\min_{j, \alpha_{j}\neq 0} \alpha_{j}^{2}},1\right\}$,
\begin{align}
    \left\|\frac{w^{+}}{\mathbb{E}_{r\sim\nu}(r)} - w_{t} \right\| \leq \frac{4\beta_{t}^{L}\left( \mathbb{E}_{r\sim\nu}(r)\| \sigma^{-2}\cdot\alpha\|_{2} + \|\sigma^{-2}\|\right)\|\sigma^{-2}\cdot\alpha\|}{\mathbb{E}_{r\sim\nu}(r)^{3}\|\sigma^{-1}\cdot\alpha\|^{2}\| \sigma^{-1}\cdot\alpha\|^{2}}
    + \frac{4\beta_{t}^{L}}{\mathbb{E}_{r\sim\nu}(r)^{2}}\left(\frac{\| \sigma^{-1}\|_{4}}{\| \sigma^{-1}\cdot\alpha\|_{2}}\right)^{2}
\end{align}
But using the definition of $\beta_{t}^{L}$:
\begin{equation}
    \begin{aligned}
        \sum_{t=t_{0}+ 1}^{T} \beta_{t}^{L} &\leq \left(C\sqrt{2\ln\left(\frac{16JT}{\delta}\right)} + 4\|\sigma\|_{\infty}\sqrt{J} + 2\|\sigma\|_{\infty}\sqrt{\ln\left(\frac{8TJ}{\delta}\right)}\right) \sum_{t =t_{0} +1}^{T} \frac{1}{\sqrt{\sum_{l=1}^{t-1} K_{l}}}\\
        &\leq \left(C\sqrt{2\ln\left(\frac{16JT}{\delta}\right)} + 4\|\sigma\|_{\infty}\sqrt{J} + 2\|\sigma\|_{\infty}\sqrt{\ln\left(\frac{8TJ}{\delta}\right)}\right) \sqrt{\frac{T\sum_{t = 1}^{T-1}\frac{1}{t}}{\bar{K}_{T}}}\\
        &\leq \left(C\sqrt{2\ln\left(\frac{16JT}{\delta}\right)} + 4\|\sigma\|_{\infty}\sqrt{J} + 2\|\sigma\|_{\infty}\sqrt{\ln\left(\frac{8TJ}{\delta}\right)}\right) \sqrt{\frac{T\ln(T)}{\bar{K}_{T}}}\\
    \end{aligned}
\end{equation}
where $\bar{K}_{T} := \frac{T}{\sum_{t=t_{0}}^{T-1}\frac{t-1}{\sum_{l=1}^{t-1} K_{l}}}$ is the average number of arms presented over the $T$ steps, $K_{T} \geq K$.
Therefore the regret is bounded with probability at least $1 - \delta$
\begin{align}
    R_{T} \leq \frac{2}{\|\sigma^{-1}\cdot \alpha\|}\left(1 + \frac{2L^{2}}{K}\max\left\{\frac{64}{\mathbb{E}_{r\sim\nu}(r)^{2}\min_{j, \alpha_{j}\neq 0} \alpha_{j}^{2}},1\right\} \right) + \frac{16LS\Phi}{\mathbb{E}_{r\sim\nu}(r)}\max\left\{1,  \frac{1}{\mathbb{E}_{r\sim\nu}(r)}\right\}\sqrt{\frac{T\ln(T)}{\bar{K}_{T}}}
\end{align}
where $S = \left( \| \alpha\cdot \sigma^{-2}\|_{2} + \|\sigma^{-2}\|_{2}\right)\frac{\|\sigma^{-2}\cdot\alpha\|_{2}}{\| \sigma^{-1}\cdot\alpha\|_{2}^{4}} + \left(\frac{\| \sigma^{-1}\|_{4}}{\| \sigma^{-1}\cdot\alpha\|_{2}}\right)^{2}$ and
$L = C\sqrt{2\ln\left(\frac{16JT}{\delta}\right)} + 8\|\sigma\|_{\infty}\sqrt{J} + 4\|\sigma\|_{\infty}\sqrt{\ln\left(\frac{8TJ}{\delta}\right)}$

\section{REGRET ANALYSIS}\label{app:regret}

In \Secref{sec:preliminaries}, the regret has been defined as the difference between the top-$K$ arms according to the oracle and the $K$ arms selected by a learner, see \Eqref{eq:regret.definition}.
However one may argue that a more interesting notion of regret is to compare the actions of the oracle to the action of the learner with respect to the true reward
of each arm. That is to say, to define the notion of \emph{absolute} regret,
\begin{equation}\label{eq:true.regret}
  R_{T}^{\text{abs}} = \sum_{t=1}^{T} \sum_{i\in\gA_{t}^{+}} r_{i, t} - \sum_{i\in\gA_{t}} r_{i, t}
\end{equation}
The two notions of regret, that is to say the \emph{relative} and \emph{absolute} regret are related,
\begin{equation}
    \begin{aligned}
        R_{T} - R_{T}^{\text{abs}} = \sum_{t=1}^{T} \sum_{i\in\gA_{t}^{+}} \langle w^{+}, g^{-1}(\phi_{i,t}) - g^{-1}(g(\alpha\cdot r_{i,t})) \rangle - \sum_{i\in\gA_{t}} \langle w^{+}, g^{-1}(\phi_{i,t}) - g^{-1}(g(\alpha\cdot r_{i,t})) \rangle
    \end{aligned}
\end{equation}
However controlling the deviation here is not possible through standard concentration inequalities. Indeed the set of arms, $\gA_{t}^{+}$ is computed by taking the argmax over the noise therefore $(\epsilon_{i,t})_{i\in\gA_{t}^{+}}$ are not centered anymore, see App.~\ref{app:noise.correlation}.
Nonetheless, one can still provide some guarantees on the regret $\tilde{R}_{T}^{\text{abs}}$.

\begin{theorem}\label{thm:regret.absolute}
    For every $t\leq T$ let $(\phi_{(i),t})_{i\leq K_{t}}$ be the ordered set of arms with respect to the oracle $\wh{\mathfrak{O}}$,
    that is to say $\langle w^{\star}, g^{-1}(\phi_{t,(1)})\rangle \geq \langle w^{\star}, g^{-1}(\phi_{t,(2)})\rangle \geq \hdots \geq \langle w^{\star}, g^{-1}(\phi_{t,(K_{t})})\rangle$
    let's now define the gap between the top-$K$ according the oracle and the rest of the arms, $\Delta_{t,K} := \langle w^{+}, g^{-1}(\phi_{(K),t}) - g^{-1}(\phi_{(K+1), t})\rangle$. For any weights
    $w\in\mathbb{R}^{J}$ such that:
    \begin{equation}
        \forall i\leq K_{t}, \qquad \left| \langle w^{+} - w, g^{-1}(\phi_{i,t})\rangle \right| \leq \frac{\Delta_{t,K}}{4}
    \end{equation}
    Then the top-$K$ arms, $\gA := \arg\max_{i}^{K} \langle w, \phi_{i,t}\rangle$ is exactly the ranking of the oracle, that is to say $\gA = \gA_{t}^{+}$.
\end{theorem}
\begin{proof}
 For an arm $i\in \gA_{t}^{+}$,
 \begin{align*}
    \langle w_{t}, g^{-1}(\phi_{i,t})\rangle \geq \langle w^{\star}, g^{-1}(\phi_{i,t})\rangle - \frac{\Delta_{t, K}}{4} &\geq \langle w^{\star}, g^{-1}(\phi_{(K),t})\rangle - \frac{\Delta_{t,K}}{4}\\
    &\geq \langle w^{\star}, g^{-1}(\phi_{(K+1),t}))\rangle + \frac{3\Delta_{t,K}}{4}\\
    &\geq \langle w_{t}, g^{-1}(\phi_{(K+1),t}))\rangle + \frac{\Delta_{t,K}}{2} > \langle w_{t}, g^{-1}(\phi_{(K+1),t}))\rangle
\end{align*}
Therefore $\gA_{t}^{+}\subset \gA$. On the other hand for any $i\not\in \gA_{t}^{+}$, we have that:
\begin{align*}
    \langle w_{t}, g^{-1}(\phi_{i,t})\rangle \leq \langle w^{\star}, g^{-1}(\phi_{(K+1),t}))\rangle + \frac{\Delta_{t,K}}{4} &\leq \langle w^{\star}, g^{-1}(\phi_{(K),t})\rangle - \frac{3\Delta_{t,K}}{4}\\
    &\leq \langle w_{t}, g^{-1}(\phi_{(K),t})\rangle - \frac{\Delta_{t,K}}{2} < \langle w_{t}, g^{-1}(\phi_{(K),t})\rangle
\end{align*}
Therefore, the set $\gA_{t}^{+} \subset \gA$ but because both are of size $K$, we have that $\gA = \gA_{t}^{+}$.
\end{proof}

The main implication of Thm.~\ref{thm:regret.absolute} is that we can now provide a bound on the absolute regret of \textsc{GLM-$\varepsilon$-greedy}.
Indeed, thanks to the bound of Lem.~\ref{lem:error.alpha}, for $t\geq \max\left\{\left(\frac{2C^{2}\sqrt{2\ln\left(\frac{8T}{\delta}\right)}}{\varepsilon K \eta_{\nu}^{2}}\right)^{2}, \frac{32\left(\sqrt{\lambda}\|\alpha\| + \|\sigma\|_{\infty}\left( \sqrt{\ln\left(1 + \frac{2T}{\lambda}\right) + \ln\left(\frac{8JT}{\delta}\right)}\right) \right)^{2}}{\varepsilon K \eta_{\nu}^{2} \min_{t\leq T} \Delta_{t,K}^{2}}\right\}$, the error between the estimated weights $w_{t}$ and the optimal oracle weights $w^{+}$ is bounded by $\Delta_{t,K}/4$,
therefore the absolute regret of Alg.~\ref{alg:eps.greedy.glm} is bounded by
\begin{align}
    R_{T}^{\text{abs}} \leq C\max\left\{\left(\frac{2C^{2}\sqrt{2\ln\left(\frac{8T}{\delta}\right)}}{\varepsilon K \eta_{\nu}^{2}}\right)^{2}, \frac{32\left(\sqrt{\lambda}\|\alpha\| + \|\sigma\|_{\infty}\left( \sqrt{\ln\left(1 + \frac{2T}{\lambda}\right) + \ln\left(\frac{8JT}{\delta}\right)}\right) \right)^{2}}{\varepsilon K \eta_{\nu}^{2} \min_{t\leq T} \Delta_{t,K}^{2}}\right\} &\\
    + \frac{4\Phi}{\left\| \alpha\cdot\sigma^{-1}\right\|}\left(\sqrt{2T\ln\left(\frac{4T}{\delta}\right)} +  \varepsilon T\right)&
\end{align}

Choosing $\varepsilon = \min\left\{ T^{-1/3}, \min_{t\leq T} \Delta_{t,K}^{2}, \frac{1}{\sqrt{T}\min_{t\leq T} \Delta_{t,K}}\right\}$ yields an absolute regret of $R_{T}^{\text{abs}} = \wt{\mathcal{O}}\left(T^{2/3}, \frac{\sqrt{T}}{\min_{t\leq T} \Delta_{t,K}}, \frac{1}{\min_{t\leq T} \Delta_{t,K}^{4}} \right)$.
Therefore, when $\min_{t\leq T} \Delta_{t,K} \geq T^{-1/6}$ the absolute regret and relative regret of Alg.~\ref{alg:eps.greedy.glm} scales similarly with a relative and absolute regret of order $\wt{\mathcal{O}}(T^{2/3})$.

\section{ADDITIONAL EXPERIMENTS}\label{app:experiments}

In this section, we present in details the algorithms used in Sec.~\ref{sec:experiments}, the experimental protocol and additionnal experimental results.

        \begin{algorithm}[t]
            \small
            \DontPrintSemicolon
            \caption{\textsc{GLM-$\varepsilon$-greedy-ALL SAMPLES} algorithm}
            \label{alg:eps.greedy.glm.all}
            \KwIn{Noise parameters $\{\sigma_j\}_{j\leq J}$, confidence level $\delta$}
            \textbf{Parameters:} exploration level $\varepsilon$; number of arms to pull $K$; regularization $\lambda$\;
            Set $\mathcal{H}_{0} = \emptyset$, $\wh{\alpha} = 0$ and $w_{0} = 0$\;
            \For{$t = 1, \dots, T$}{
                Sample $Z_{t} \sim \text{Ber}\left( \varepsilon\right)$\;
                Observe evaluations for each arm $(\phi_{i,t})_{i\leq K_{t}}$\;
                \If{$Z_{t} = 1$}
                {
                    Pull arms in $\gA_t$ obtained by sampling $K$ arms uniformly in $\{1, \ldots, K_{t} \}$\;
                    Observe rewards $r_{i,t}$ for all $i\in \gA_{t}$\;
                }
                \Else{
                    Select $\gA_{t} = \arg\max_{i}^{K}\langle w_{t}, g^{-1}(\phi_{i,t})\rangle$\;
                }
                Add sample to dataset $\mathcal{H}_{t} = \mathcal{H}_{t-1} \cup \big(\cup_{i\in \gA_{t}}\{ (\phi_{i,t}, r_{i,t})\}\big)$\;
                Update estimators $\widehat{\alpha}_{j,t}$ by solving
                \begin{equation}
                        \sum_{\phi,r \in \mathcal{H}_{t}} r(g(\wh{\alpha}_{t,j} \cdot r) - \phi_{j}) - \lambda \wh{\alpha}_{t,j} = 0
                \end{equation}
                Update weights $w_{t,j} = \wh{\alpha}_{t,j}/(\sigma_{j}^{2})\|\wh{\alpha}_t\cdot\sigma^{-1}\|^{2}$
            }
        \end{algorithm}

            \begin{algorithm}[t]
                \small
                \DontPrintSemicolon
                \caption{\textsc{GLM-EvalBasedUCB} algorithm}
                \label{alg:eval.based.ucb}
                \KwIn{Noise parameters $\{\sigma_j\}_{j\leq J}$, confidence level $\delta$}
                Set $\mathcal{H}_{0} = \emptyset$, $\wh{\alpha} = 0$ and $w_{0} = 0$\;
                \For{$t = 1, \dots, T$}{
                    Observe evaluations for each arm $(\phi_{i,t})_{i\leq K_{t}}$\;
                    Compute $\wh{\alpha}_{t} = \frac{\sum_{l=1}^{t-1} \sum_{i\in \gA_{l}}g^{-1}(\phi_{i,l})}{\sum_{l=1}^{t-1} \sum_{i\in \gA_{l}}r_{i,l}}$ and
                    \begin{align}\label{eq:beta_eval}
                        \beta_{t}^{E} = \frac{\frac{2\sqrt{2J\ln(2/\delta)}}{3} + \sqrt{(t-1)K\sum_{j=1}^{J}\sigma_{j}^{2}\ln(2/\delta)}}{\sum_{l=1}^{t-1}\sum_{i\in \gA_{t}} r_{i,l}}
                    \end{align}
                    \If{$\|\wh{\alpha}_t\|_2 \leq \beta_t^{E}$}{
                        Compute weights $w_{t} = \frac{\sigma^{-2}\hat{\alpha}_{t}}{\| \sigma^{-1}\hat{\alpha}_{t}\|^{2}}$\;
                    }
                    \Else{
                        Compute $\lambda_t >0$ solving $\sum_{j=1, \sigma_{j} >0}^{J} \left(\frac{\wh{\alpha}_{t,j}}{\lambda_{t}\sigma_{j}^{2} + (\beta_{t}^{E})^{2}}\right)^{2} = \frac{1}{\beta_{t}^{2}}$\;
                        Compute $\wt{\alpha}_{t} = \left(I_{J} - \beta_{t}^{2}\left(\lambda_{t}I_{J} + (\beta_{t}^{E})^{2}\sigma^{-2} \right)^{-1}\sigma^{-2} \right)\wh{\alpha}_{t}$ and  $w_{t} = \frac{\sigma^{-2}\wt{\alpha}_{t}}{\| \sigma^{-1}\wt{\alpha}_{t}\|_{2}}$\;
                    }
                    Select $\gA_{t} = \arg\max_{i}^{K}\langle w_{t}, g^{-1}(\phi_{i,t})\rangle$ and observe rewards $r_{i,t}$ for all $i\in \gA_{t}$\;
                }
            \end{algorithm}

        \begin{algorithm}[t]
            \small
            \DontPrintSemicolon
            \caption{\textsc{Rand} algorithm}
            \label{alg:rand}
            \textbf{Parameters:} number of arms to pull $K$\;
            \For{$t = 1, \dots, T$}{
                Randomly select $j\leq J$\;
                Observe evaluations for each arm $(\phi_{i,t})_{i\leq K_{t}}$\;
                Select $\gA_{t} = \arg\max_{i}^K g^{-1}(\phi_{i,t,j})$
            }
        \end{algorithm}

\subsection{Baseline Algorithms}

    In Sec.~\ref{sec:experiments}, we compare our algorithms, \textsc{GLM-$\varepsilon$-greedy} (Alg.~\ref{alg:eps.greedy.glm}) and \textsc{ESAG} (Alg.~\ref{alg:greedy.affine}),
    to different baselines.

    \paragraph{\textsc{GLM-$\varepsilon$-greedy-ALL} (Alg.~\ref{alg:eps.greedy.glm.all}).}

        This algorithm is the s Alg.~\ref{alg:eps.greedy.glm} and illustrate the problem with the noise correlation of App.~\ref{app:noise.correlation}. This algorithm thus updates its own
        MLE estimator at every time step instead of step with a random exploration.

        \paragraph{\textsc{GLM-EvalBasedUCB} (Alg.~\ref{alg:eval.based.ucb}).}

            This algorithm does not rely on an $\varepsilon$-greedy type of exploration but on an optimistic approach. Similar to other algorithms, it first computes an MLE estimator of $\alpha$ using all samples collected over time. Then it build a ``confidence'' set around the estimated $\wh{\alpha}$ and it computes weights $w_t$ as those resulting from a worst-case choice of a parameter $\wt\alpha$ in terms of error. Notice that the confidence sets are not theoretically justified because they are designed by ignoring the correlation between decisions and evaluations. As such, this algorithm has not theoretical guarantee and it should be considered as a heuristic. We provide further details in Sect.~\ref{ssec:derivation.evalubased}.

        \paragraph{\textsc{RAND} (Alg.~\ref{alg:rand}).}

        This baseline algorithm is simply selecting one random evaluators every step and computing a ranking of the arms based on this random evaluators.

        \paragraph{\textsc{GLM-Greedy}.} This algorithm is a variant of \textsc{GLM-$\varepsilon$-Greedy-ALL} with $\varepsilon = 0$.

        \subsection{Derivation of \textsc{GLM-EvalBasedUCB}}\label{ssec:derivation.evalubased}
            
            Here, we breifly present how we derive the \textsc{GLM-EvalBasedUCB} algorithm. This algorithm takes a \emph{optimistic approach} w.r.t.\ the set of plausible $\alpha$ values and compute a set of weights $w_{t}$ as the set of oracle weights but for a different value of $\alpha$. The first step is to compute an estimator $\wt{\alpha}_{t}$ such that:
        \begin{align}
            \tilde{\alpha}_{t} = \arg\max_{z \in\mathbb{R}^{J} : \|\hat{\alpha}_{t} - z\| \leq \beta^{E}_{t}} \min_{w\in \mathbb{R}^{J}, \langle w, z\rangle = 1} \sum_{j=1}^{J} w_{j}^{2}\sigma_{j}^{2}
        \end{align}
        where $\beta_{t}^{E}$ is the width of the confidence intervals defined in Eq.~(\ref{eq:ap2.ci.alpha})
        Solving the minimization problem in the previous equation implies that $\wt{\alpha}_{t}$ is solution to the following equivalent optimization problem,
        \begin{align}\label{eq:equivalent_problem}
            \min_{\| z - \hat{\alpha}_{t}\|_{2} \leq \beta_t^{E}} \| \sigma^{-1} \cdot z\|_{2}^{2} = \min_{\| u \| \leq 1} \|\sigma^{-1}\cdot(\hat{\alpha}_{t} + \beta_{t}^{E}u) \|^{2}_2
        \end{align}
        with $\wh{\alpha}_{t} = \frac{\sum_{l=1}^{t-1} \sum_{i\in \gA_{l}}g^{-1}(\phi_{i,l})}{\sum_{l=1}^{t-1} \sum_{i\in \gA_{l}}r_{i,l}}$ the average features observed until time $t$. Eq.~(\ref{eq:equivalent_problem}) is a convex problem and using KKT conditions, we distinguish two different situations:
        \begin{enumerate}
            \item If $\| \hat{\alpha}_{t}\| \leq \beta_{t}^{E}$ then $\min_{\| z - \hat{\alpha}_{t}\|_{2} \leq \beta^{E}_{t}} \| \sigma^{-1} \cdot z\|_{2}^{2} = 0$ thus the problem is ill-defined and the weights are also not defined (A necessary condition for this condition to be fulfilled
            is that $\| \alpha \| \leq \beta_{t}^{E}( 1 + \sqrt{1 + (\beta_{t}^{E})^{2}})$)
            \item If $\| \hat{\alpha}_{t}\| > \beta_{t}^{E}$ then the solution is such that $u = - \beta_{t}^{E}\left(\lambda_{t}I_{J} + (\beta_{t}^{E})^{2}\text{Diag}(\sigma^{-2}) \right)^{-1}\sigma^{-2}\cdot\hat{\alpha}_{t}$ where $\lambda_{t} > 0$ is solution to\footnote{A solution $\lambda_{t}$ always exist when there is at least one evaluator $j$ such that $\sigma_{j}>0$ because the function $f : \lambda \mapsto \sum_{j=1, \sigma_{j}>0} \left(\frac{\hat{\alpha}_{t,j}}{\lambda \sigma_{j}^{2} + (\beta_{t}^{E})^{2}}\right)^{2}$ is strictly non-increasing over $\mathbb{R}^{+}$ and for $\lambda = 0$, we have $f(0) = \| \hat{\alpha}_{t}\|^{2}/(\beta_{t}^{E})^{4}> (\beta_{t}^{E})^{-2}$ and $\lim_{\lambda\rightarrow +\infty} f(\lambda) = 0$ so by continuity there exists a unique solution $\lambda_{t}>0$.}:
            \begin{align}
                \sum_{j=1, \sigma_{j} >0}^{J} \left(\frac{\hat{\alpha}_{t,j}}{\lambda_{t}\sigma_{j}^{2} + (\beta_{t}^{E})^{2}}\right)^{2} = \frac{1}{(\beta_{t}^{E})^{2}}
            \end{align}
            Overall, $\wt{\alpha}_{t}$ is defined as:
            \begin{equation}
                \wt{\alpha}_{t} = \left(I_{J} - (\beta_{t}^{E})^{2}\left(\lambda_{t}I_{J} + (\beta_{t}^{E})^{2}\text{Diag}(\sigma^{-2}) \right)^{-1}\sigma^{-2}\right)\hat{\alpha}_{t},
            \end{equation} 
            that is to say the weights used by the algorithm is $w_{t} = \frac{\Sigma^{-2}\tilde{\alpha}_{t}}{\| \Sigma^{-1}\tilde{\alpha}_{t}\|_{2}^{2}}$
        \end{enumerate}
        \subsection{Additional Synthetic Experiments}

            We report here the regret and estimation bias for the in the logisitic and linear setting for the three different regimes of $\frac{\alpha_{0}}{\sigma_{0}}\in\{0.1, 1, 10\}$ as reported in \Secref{sec:experiments}.
            For both the linear and logistic case we choose $K = 10$ and $K_{t} = 60$. For \textsc{GLM-$\varepsilon$-Greedy} and \textsc{GLM-$\varepsilon$-greedy-ALL SAMPLES}, we use different values of $\varepsilon\in \{0.1 T^{-1/3}, 0.1 T^{-1/2}\}$.
            The rewards are sampled from a centered unit variance normal distribution truncated between $[0,20]$.

            \paragraph{Logistic and Linear Experiments.}
             As highlighted in \Secref{sec:experiments}, for a small ratio $\frac{\alpha_{0}}{\sigma_{0}}$, the algorithms  \textsc{GLM-$\varepsilon$-greedy-ALL SAMPLES} or \textsc{GLM-EvalBasedUCB} that use all samples
             suffer from a linear regret compared to our algorithm Alg.~\ref{alg:eps.greedy.glm} for $\varepsilon = \mathcal{O}(T^{-1/3})$. In addition, we observe that
             as the ratio $\frac{\alpha_{0}}{\sigma_{0}}$ increases the correlation effect highlighted in App.~\ref{app:noise.correlation} becomes less and less relevant and \textsc{GLM-EvalBasedUCB} performs particularly well.
            \begin{figure}
                \centering
                \includegraphics[width=0.8\textwidth]{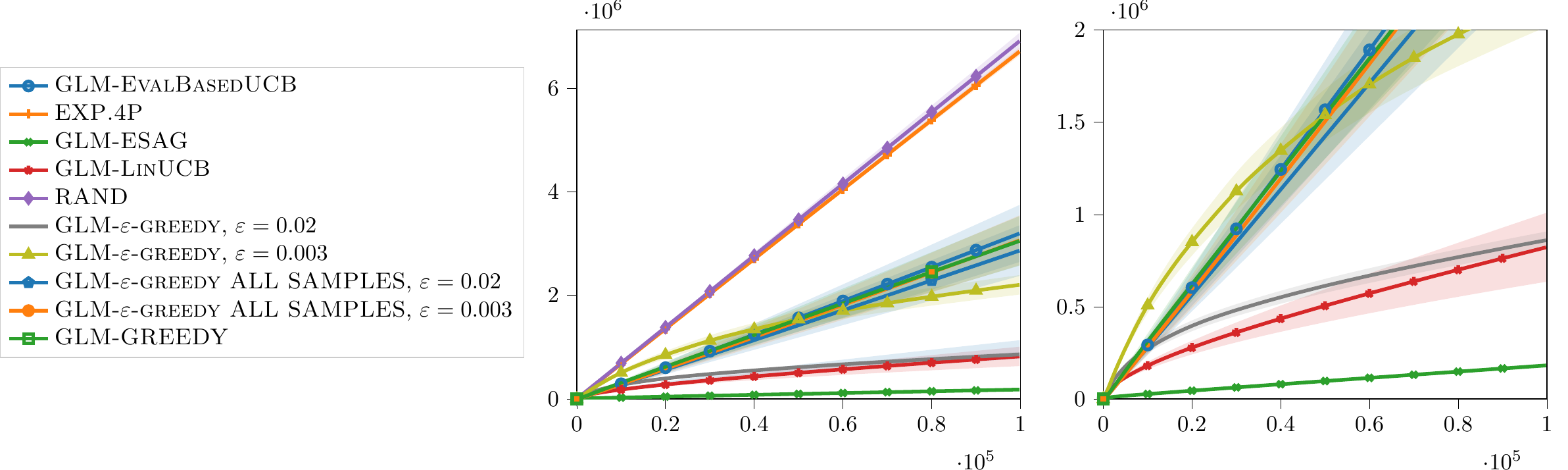}
                \caption{$\frac{\alpha_{0}}{\sigma_{0}} = 0.1$}
                \label{fig:logistic_regret_ratio_0_1}
                \centering
                \includegraphics[width=0.8\textwidth]{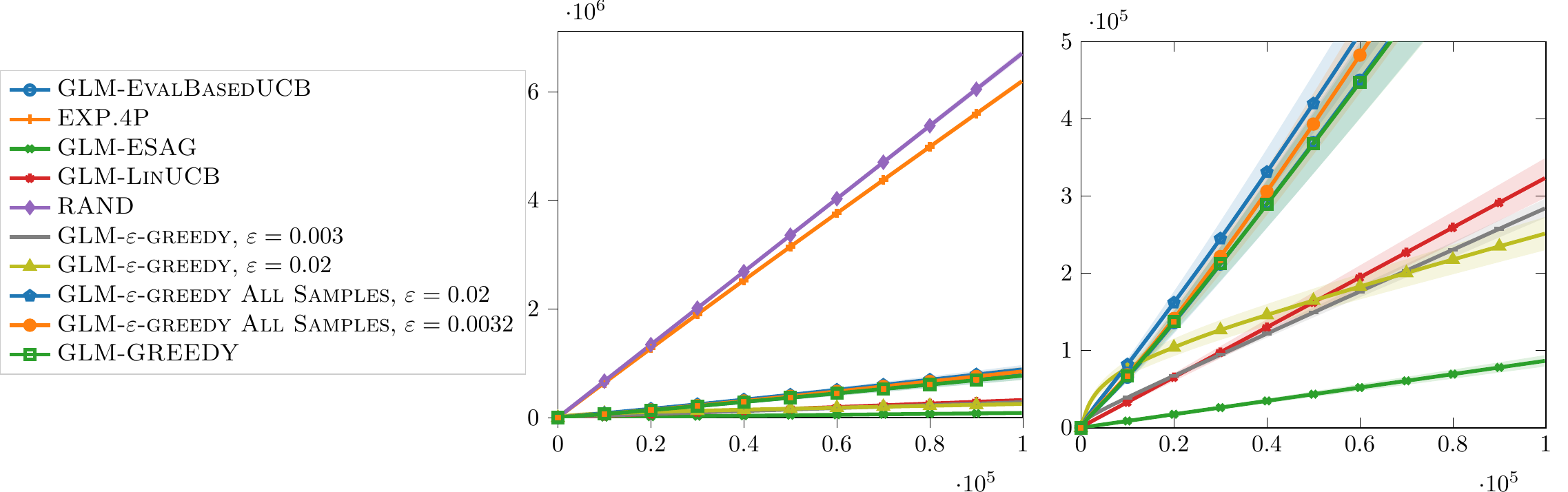}
                \caption{$\frac{\alpha_{0}}{\sigma_{0}} = 1$}
                \label{fig:logistic_regret_ratio_1}

                    \centering
                    \includegraphics[width=0.8\textwidth]{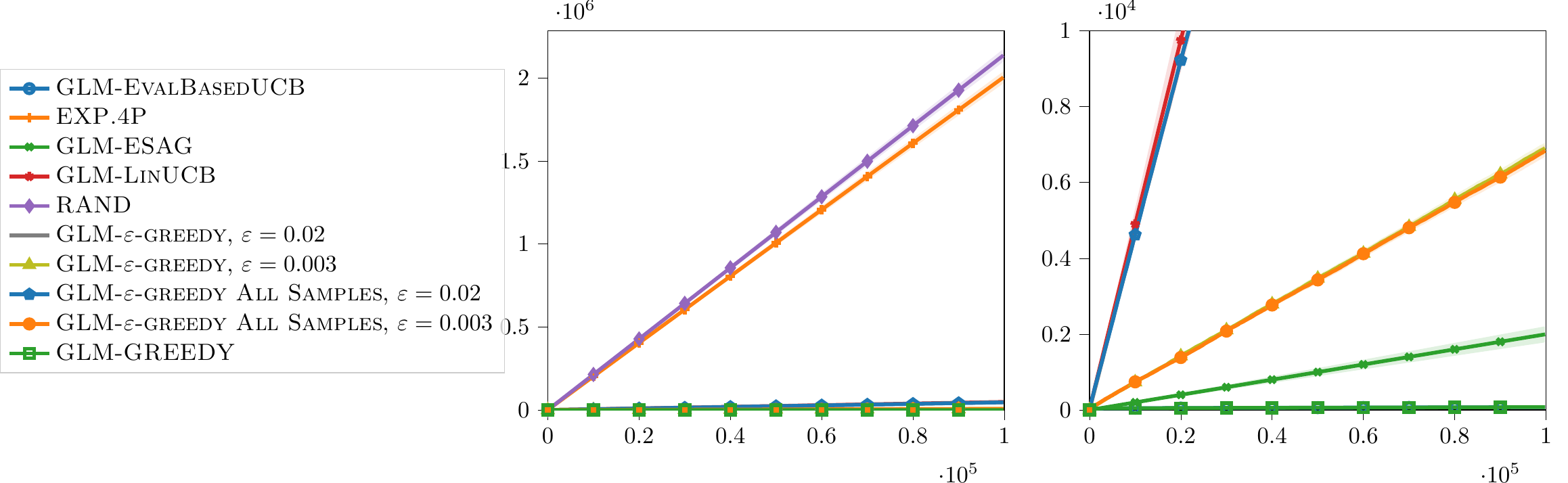}
                    \caption{$\frac{\alpha_{0}}{\sigma_{0}} = 10$}
                    \label{fig:logistic_regret_ratio_10}
                   \caption{Regret (\emph{Left}) and zoomed in regret (\emph{Right}, removing \textsc{RAND} and \textsc{EXP4.P} algorithms) wrt to the oracle $\mathfrak{O}$ in \Secref{sec:general.case} for different values of $\frac{\alpha_{0}}{\sigma_{0}}$}
                   \label{fig:logistic_regret}
           \end{figure}

           \begin{figure}
            \centering
            \includegraphics[width=0.7\textwidth]{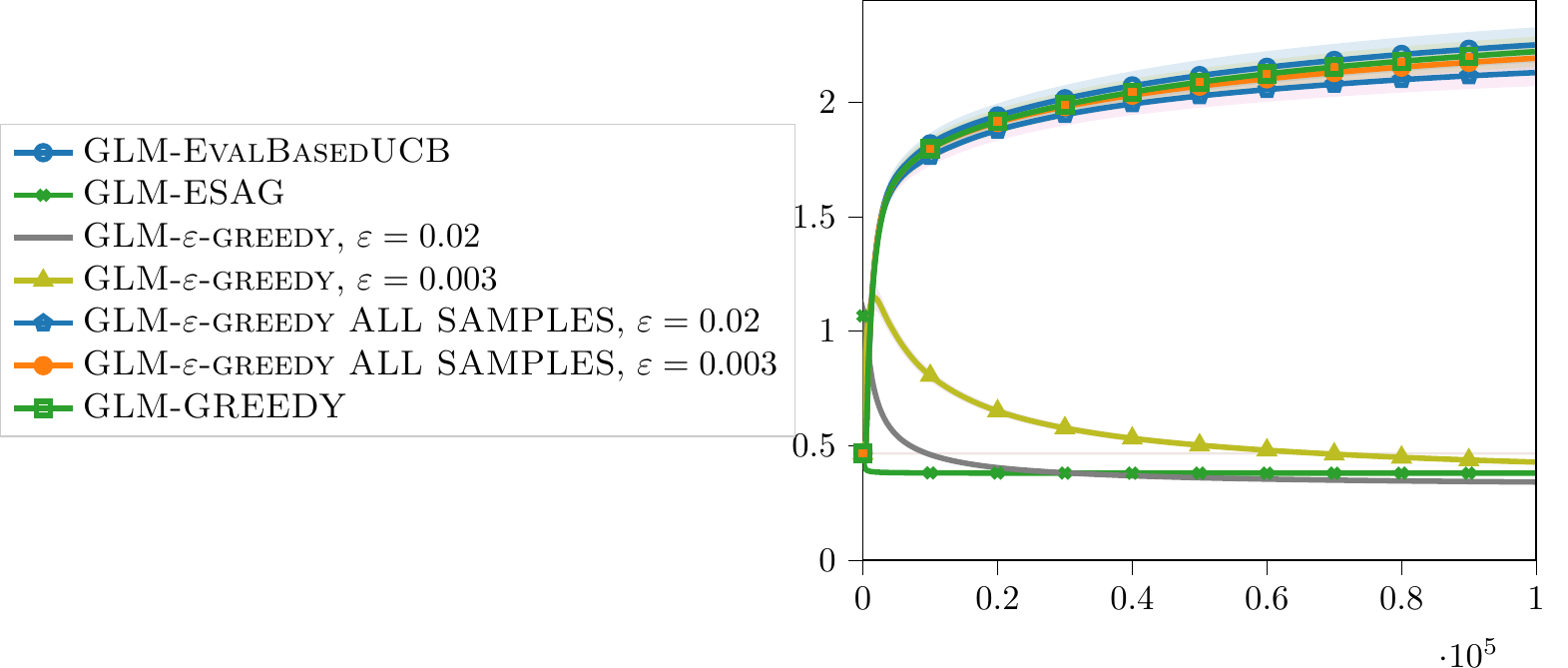}
            \caption{$\frac{\alpha_{0}}{\sigma_{0}} = 0.1$}
            \label{fig:logistic_regret_ratio_0_1}
            \centering
            \includegraphics[width=0.7\textwidth]{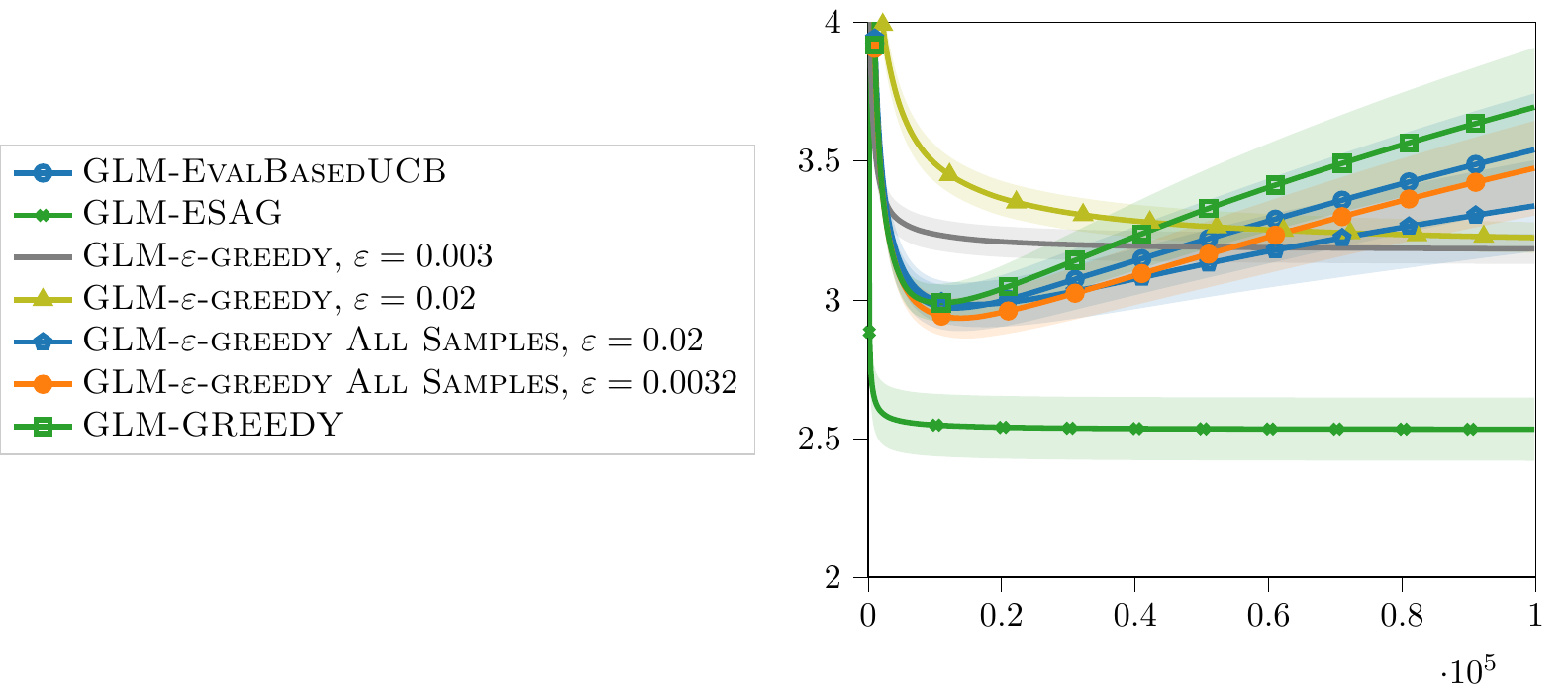}
            \caption{$\frac{\alpha_{0}}{\sigma_{0}} = 1$}
            \label{fig:logistic_regret_ratio_1}

                \centering
                \includegraphics[width=0.7\textwidth]{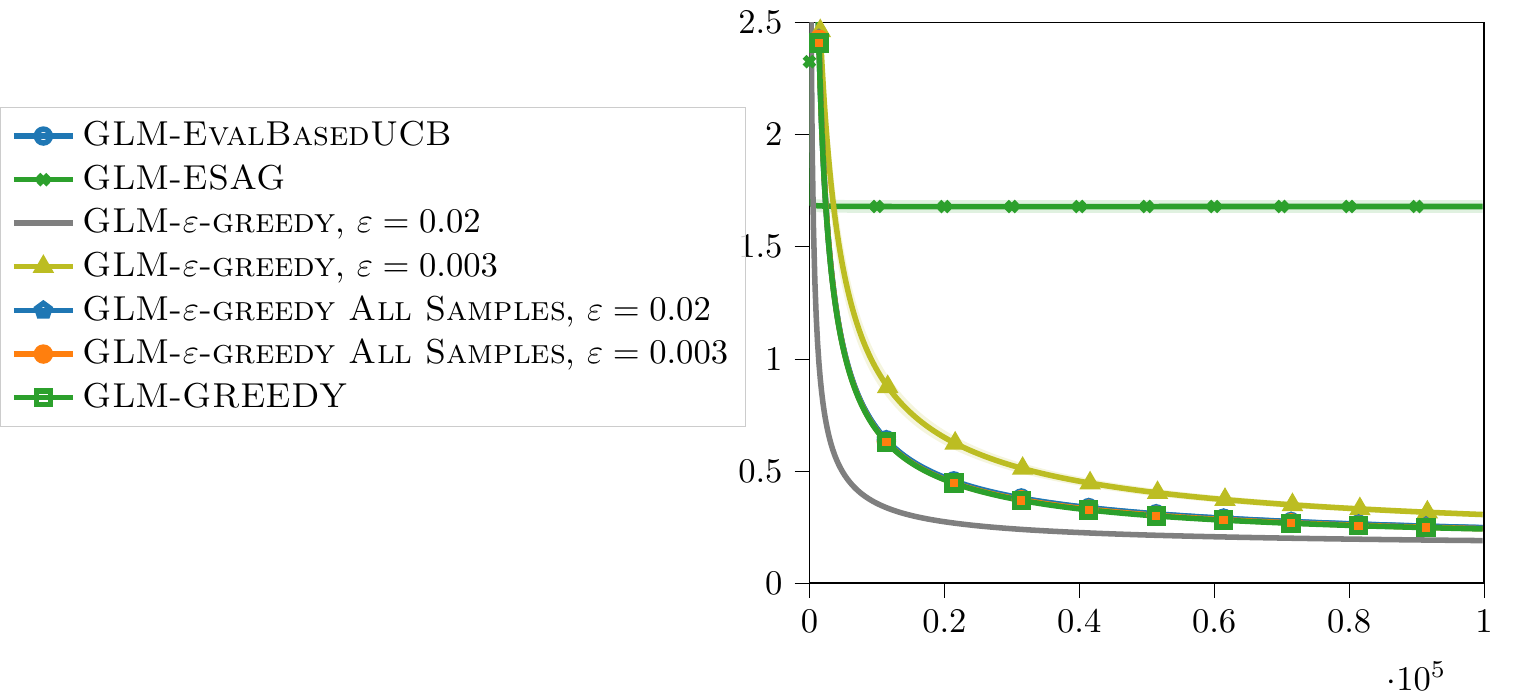}
                \caption{$\frac{\alpha_{0}}{\sigma_{0}} = 10$}
                \label{fig:logistic_regret_ratio_10}
               \caption{Estimation error for the $\alpha$ parameters for different values of $\frac{\alpha_{0}}{\sigma_{0}}$}
               \label{fig:logistic_regret}
       \end{figure}

            \begin{figure}
                \centering
                \begin{subfigure}[b]{0.3\textwidth}
                    \centering
                    \includegraphics[width=\textwidth]{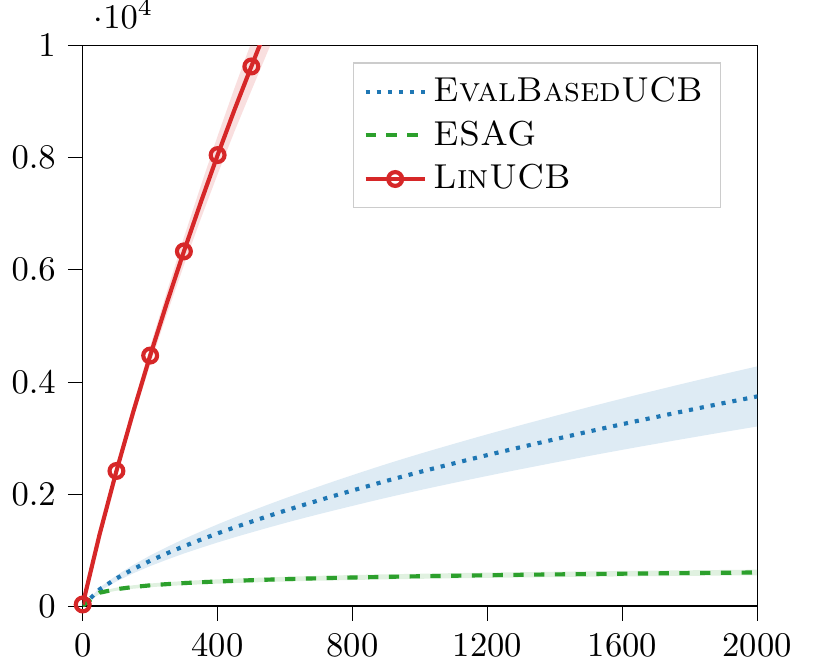}
                    \caption{$\frac{\alpha_{0}}{\sigma_{0}} = 0.1$}
                    \label{fig:linear_regret_ratio_0_1}
                \end{subfigure}
                \hfill
                \begin{subfigure}[b]{0.3\textwidth}
                   \centering
                    \includegraphics[width=\textwidth]{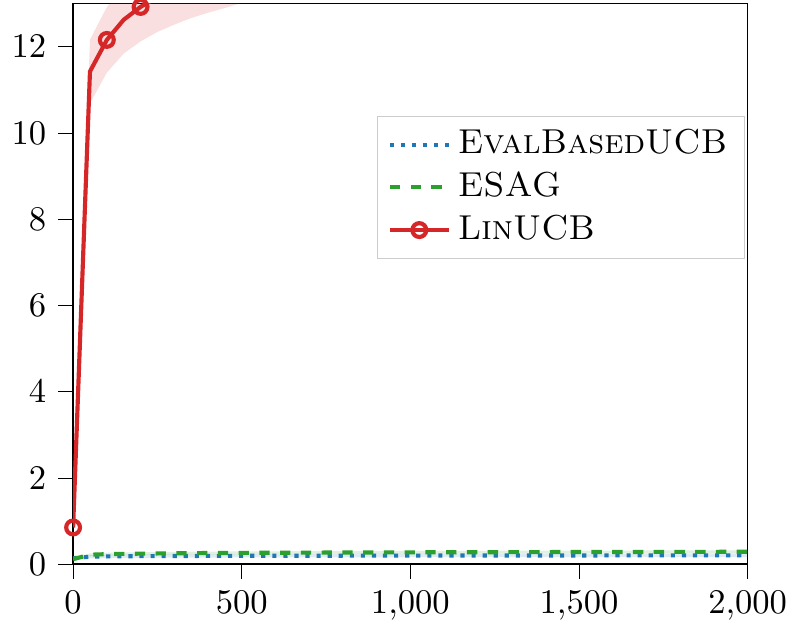}
                    \caption{$\frac{\alpha_{0}}{\sigma_{0}} = 1$}
                    \label{fig:linear_regret_ratio_1}
                \end{subfigure}
                \hfill
                \begin{subfigure}[b]{0.3\textwidth}
                    \centering
                     \includegraphics[width=\textwidth]{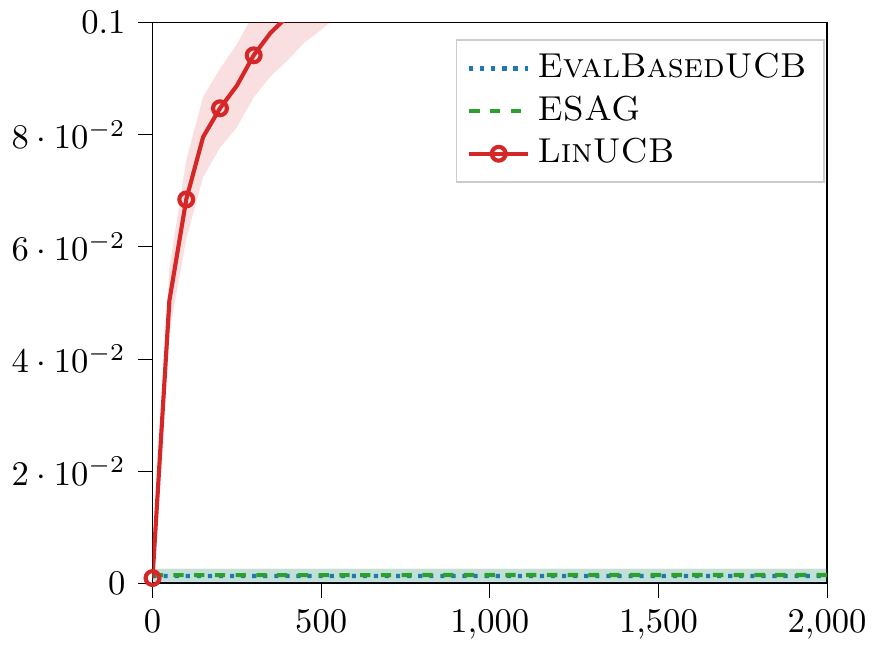}
                    \caption{$\frac{\alpha_{0}}{\sigma_{0}} = 10$}
                    \label{fig:linear_regret_ratio_10}
                \end{subfigure}
                   \caption{Regret w.r.t.\ to the oracle $\mathfrak{O}$ in \Secref{sec:affine} for different values of $\frac{\alpha_{0}}{\sigma_{0}}$}
                   \label{fig:linear_regret}
           \end{figure}

            \begin{figure}[!h]
                \centering
                \includegraphics[width=\linewidth]{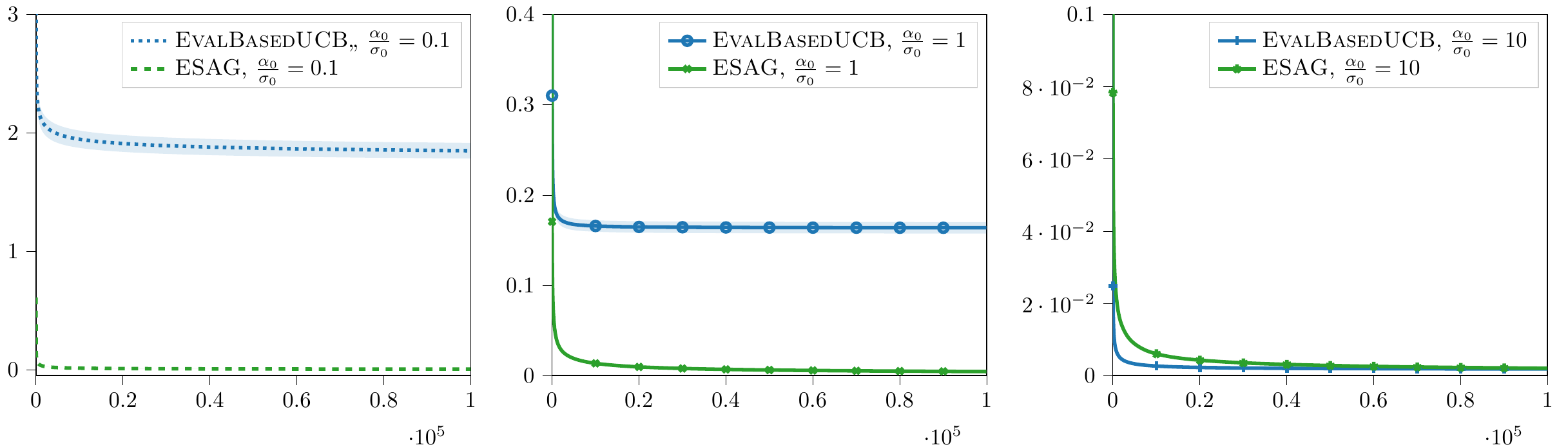}
                \caption{Estimation Bias in the linear case. For \textsc{ESAG}, we compute the error $\| \wt{\alpha}_{t} - \mathbb{E}_{r\sim\nu}(r)\alpha \|$}
            \end{figure}



        \subsection{Second Content Review Prioritization Dataset.}

            In order to validate our observations of \Secref{sec:experiments} on the content review proritization dataset, we consider a second small dataset ($D2$) of content similar to the one used in \Secref{sec:experiments}. Table~$2$ sums up the cumulative reward of each algorithms
            at $T = 2000$ steps selecting $K = 10$ out of $K_{t} = 200$ arms. Similarly to Table~$1$ linear algorithms performs best compared to the logistic algorithms with the same three algorithms getting the best performance.

            \begin{table}\label{table:results_second_d2}
                \small
                    \caption{Cumulative badness on $D2$ for $T = 2000$ steps\vspace{-0.1in}}
                    \label{fig:cumul.badness}
                    \begin{center}
                        \begin{tabular}{ |c|c| }
                        \hline
                        Alg. & Badness 
                        \\
                        \hline
                        \hline
                        \textsc{Rand} & $60025.8$ 
                        \\
                        \textsc{GLM-$\varepsilon$-greedy} & $105965.8$ 
                        \\
                        \textsc{GLM-$\varepsilon$-greedy-all} & $105499.1$ 
                        \\
                        \textsc{GLM-EvalBasedUCB} & $105545.4$
                        \\
                        \textsc{GLM-LinUCB} & $60347.4$ 
                        \\
                        \textsc{GLM-ESAG} & $106237.5$
                        \\
                        \textsc{GLM-Greedy} & $105264.3$  
                        \\
                        \textsc{Exp4.P} & $88979.4$ 
                        \\
                        \textsc{EvalBasedUCB} & $\mathbf{138230.1}$ 
                        \\
                        \textsc{LinUCB} & $\mathbf{144009.4}$ 
                        \\
                        \textsc{ESAG} &  $\mathbf{141172.1}$ 
                        \\
                        \textsc{Greedy} & $106195.4$ 
                        \\
                        \hline
                        \end{tabular}
                        \vspace{-0.2in}
                    \end{center}
                \end{table}

\subsection{Jester Experiments.}
We consider the Jester dataset~\citep{GoldbergRGP01}, which consists of joke ratings in a continuous range from $-10$ to $10$ for a total of $100$ jokes and $73421$ users. We consider the same set of users and jokes as in~\citep{RiquelmeTS18}. For a subset of $40$ jokes and $19181$ users rating all these $40$ jokes, we build evaluators as follows.
We fit $7$ \emph{contextual} models to predict the ratings from features --obtained as concatenation of user and joke features-- extracted via a low-rank factorization of the full matrix (of dimension $36$). Ratings are normalized and transformed through the logarithmic function. We trained the following models:\footnote{We used the implementations provided by scikit-learn~\citep{scikitlearn}.}
\begin{itemize}
\item rf5-s18: a random forest with 5 trees trained over 5 randomly selected features;
\item dt-s50: a decision tree with max depth $10$ trained over 50 randomly selected features;
\item linear: a linear model with intercept;
\item adaboost: an implementation of AdaBoost.R2 with 20 trees;
\item mlp: a neural network with two hidden layers ($[512,128]$) and ReLu activation;
\item adaboost20-s60: AdaBoost.R2 with 20 trees trained over 60 randomly selected features.
\end{itemize}
We use the predictions of each model as inputs for our algorithms.
As rewards, we use the real ratings and we add to them zero-mean Gaussian noise (standard deviation is $0.5$).
The resulting problem is thus misspecified since the evaluators are not perfect, as shown in Tab.~\ref{tab:jester.eval.stats}.

We estimate the parameters of the algorithms by computing statistics about the relationship between true ratings and predicted ones. For $\varepsilon$-greedy, we use the value $T^{1/3}$.
From Tab.~\ref{tab:jester.results}, we can see that linear algorithms outperform the logistic algorithms. \textsc{EvalBasedUCB} and \text{ESAG} behave similarly and perform better than \textsc{LinUCB}.

\begin{figure}[t]
\centering
\includegraphics[width=.32\textwidth]{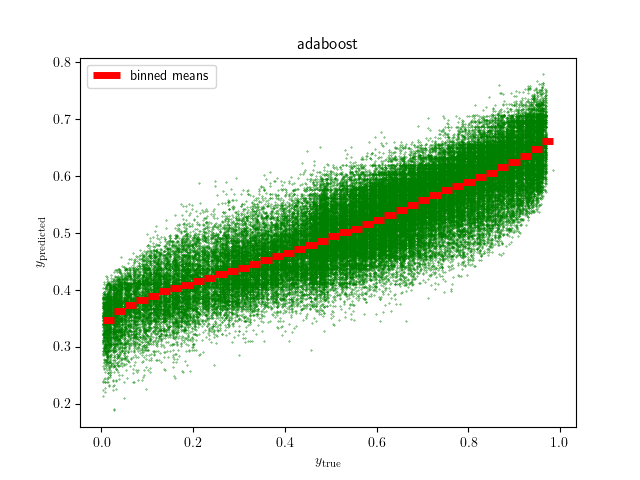}
\includegraphics[width=.32\textwidth]{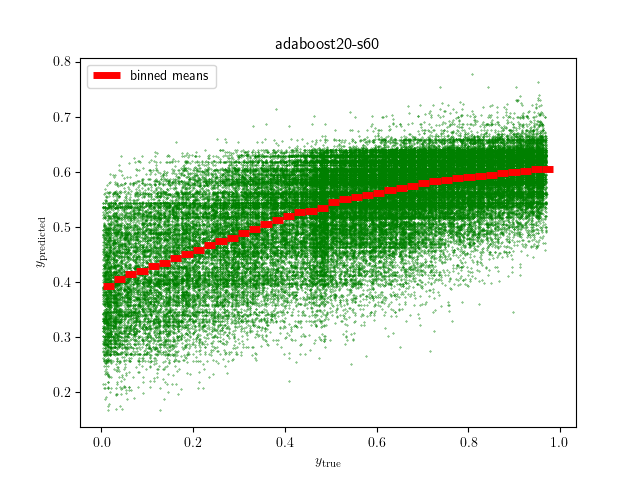}
\includegraphics[width=.32\textwidth]{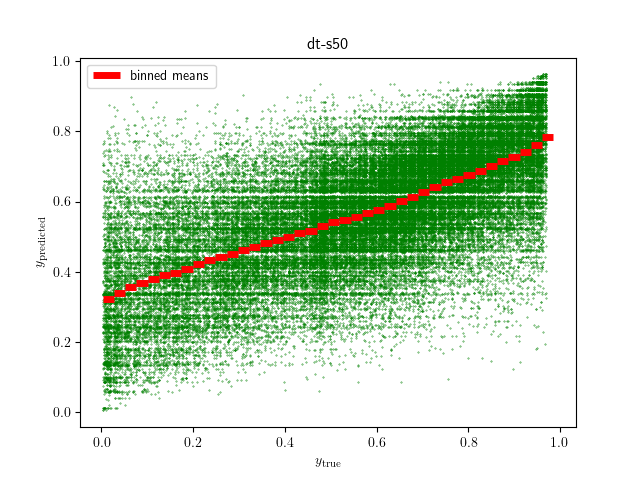}
\includegraphics[width=.32\textwidth]{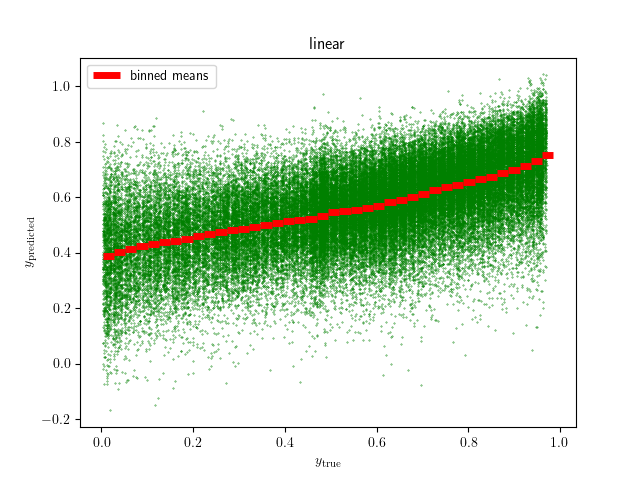}
\includegraphics[width=.32\textwidth]{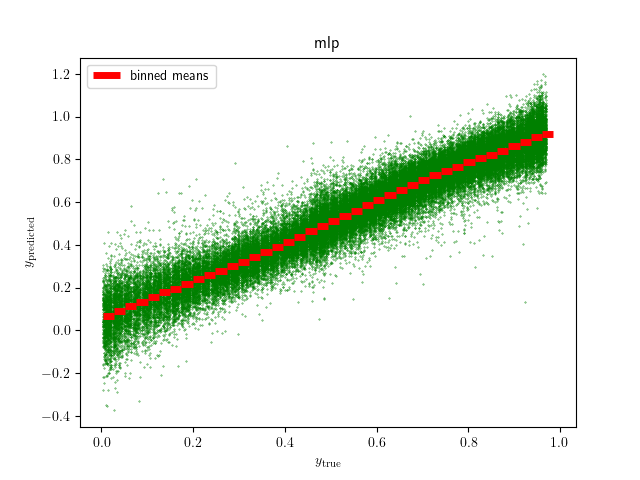}
\includegraphics[width=.32\textwidth]{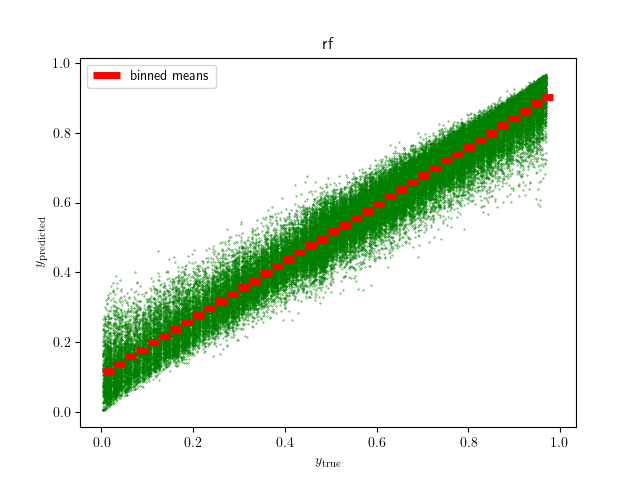}
\includegraphics[width=.32\textwidth]{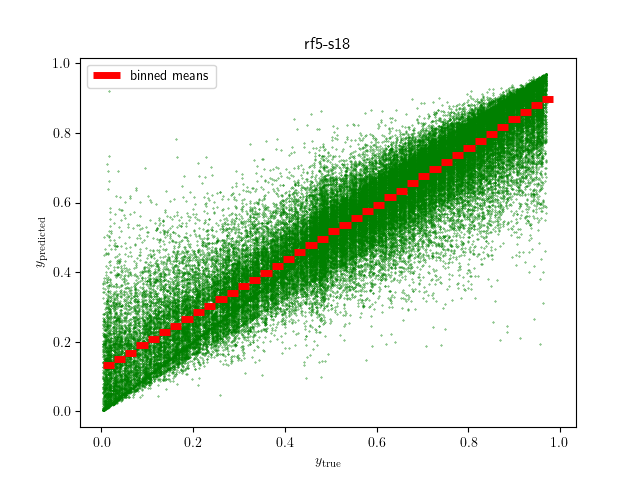}
\caption{We report the predictions of the evaluators as a function of the true ratings (for clarity we reported only $10$\% of the samples). We also report the average value over a discretization of the ratings into $40$ bins.}
\label{fig:jester.evaluators}
\end{figure}

\begin{table}[t]
    \centering
\scriptsize
    \begin{minipage}{.45\textwidth}
    \centering
    \caption{Statistics about the evaluators. We report the R2 score of the evaluators w.r.t.\ the true ratings. The column $\wh{\sigma}$ is the maximum estimate standard deviation.}
\begin{tabular}{lcc}
\toprule
          algo &    R2 &  $\wh{\sigma}$\\
\midrule
       rf5-s18 & 0.864 &          0.115 \\
        dt-s50 & 0.448 &          0.174 \\
        linear & 0.348 &          0.167 \\
      adaboost & 0.440 &          0.048 \\
           mlp & 0.917 &          0.110 \\
            rf & 0.936 &          0.064 \\
adaboost20-s60 & 0.309 &          0.086 \\
\bottomrule
\end{tabular}
    \label{tab:jester.eval.stats}

    \end{minipage}\hfill
    \begin{minipage}{.45\textwidth}
\caption{Average reward after $T=6000$ steps in the Jester experiment averaged over $50$ runs.}
\centering
\scriptsize
\begin{tabular}{lr}
\toprule
                Algorithm &    \makecell{Average Reward\\$\frac{1}{T}\sum_{i=1}^T r_i$}  \\
\midrule
  \textsc{EvalBasedUCB} & 0.87\\
  \textsc{ESAG} & 0.86\\
    \textsc{LinUCB} & 0.85\\
       \textsc{GLM-EvalBasedUCB} & 0.69\\
       \textsc{GLM-Greedy} & 0.69\\
       \textsc{GLM-ESAG}esaglogistic & 0.69\\
       \textsc{GLM-$\varepsilon$-greedy} & 0.69\\
       \textsc{RAND} & 0.57\\
       \textsc{GLM-LinUCB} & 0.51\\
\bottomrule
\end{tabular}
\label{tab:jester.results}
    \end{minipage}
\end{table}

\end{appendix}

\end{document}